\documentclass[pdflatex,sn-basic]{sn-jnl}

\usepackage{graphicx}
\usepackage{multirow}
\usepackage{amsmath,amssymb,amsfonts}
\usepackage{amsthm}
\usepackage{mathrsfs}
\usepackage[title]{appendix}
\usepackage{xcolor}
\usepackage{textcomp}
\usepackage{booktabs}
\usepackage{algorithm}
\usepackage{algpseudocode}
\usepackage{tikz}
\usepackage{comment}
\usetikzlibrary{positioning,calc,fit,backgrounds,shadows,arrows.meta,shapes.geometric}

\theoremstyle{thmstyleone}
\newtheorem{proposition}{Proposition}
\theoremstyle{thmstylethree}
\newtheorem{remark}{Remark}

\newcommand{\vect}[1]{\mathbf{#1}}
\newcommand{\mat}[1]{\mathbf{#1}}
\newcommand{\R}{\mathbb{R}}
\newcommand{\one}{\mathbf{1}}
\newcommand{\nnz}{\operatorname{nnz}}
\newcommand{\diag}{\operatorname{diag}}
\newcommand{\tr}{\operatorname{tr}}
\newcommand{\qf}{\operatorname{qf}}
\newcommand{\cosr}{\operatorname{cos}}
\newcommand{\FM}{\operatorname{FM}}
\newcommand{\MM}{\operatorname{MM}}
\newcommand{\Drop}{\operatorname{Drop}}

\begin{document}

\title[HyperFuse: fast self-supervised hypergraph node embeddings]{HyperFuse: Fast Self-Supervised Node Embeddings for Attributed Hypergraphs}

\author[1]{\fnm{Megha} \sur{P}}
\author[1]{\fnm{Harshit} \sur{Kumar}}
\author[1]{\fnm{Srajan} \sur{Agarwal}}
\author[2]{\fnm{Anirban} \sur{Banerjee}}
\author[3,4]{\fnm{Olaf} \sur{Wolkenhauer}}
\author*[1]{\fnm{Saptarshi} \sur{Bej}}\email{sbej7042@iisertvm.ac.in}

\affil[1]{\orgdiv{School of Data Science}, \orgname{Indian Institute of Science Education and Research Thiruvananthapuram}, \orgaddress{\city{Thiruvananthapuram}, \postcode{695551}, \state{Kerala}, \country{India}}}
\affil[2]{\orgdiv{Department of Mathematics and Statistics}, \orgname{Indian Institute of Science Education and Research Kolkata}, \orgaddress{\city{Mohanpur}, \postcode{741246}, \state{West Bengal}, \country{India}}}
\affil[3]{\orgdiv{Institute of Computer Science}, \orgname{University of Rostock}, \orgaddress{\city{Rostock}, \postcode{18051}, \country{Germany}}}
\affil[4]{\orgname{Stellenbosch Institute for Advanced Study (STIAS)}, \orgaddress{\city{Stellenbosch}, \country{South Africa}}}

\abstract{Self-supervised hypergraph representation learning produces informative node embeddings, but current methods train deep encoders for hundreds of epochs and can take hours on a hypergraph of a few thousand nodes. This cost becomes limiting when embeddings must be generated for many hypergraphs, for example one interaction network per patient, or repeatedly for evolving hypergraphs. We present HyperFuse, a label-free pipeline for fast embedding generation. HyperFuse (i) computes structural node coordinates by maximising a spectral relaxation of hypergraph modularity on Banerjee's hypergraph adjacency, using a matrix-free operator whose cost is linear in the number of node--hyperedge incidences; (ii) builds multi-scale feature summaries and assigns each hyperedge a bounded utility weight based on the stability of its members under feature and membership masking; and (iii) trains a small utility-weighted hypergraph encoder for 100 epochs with an invariance-plus-decorrelation objective. We compare HyperFuse with TriCL, SE-HSSL, VilLain and HypeBoy on nine public hypergraphs, using six downstream classifiers and $k$-means clustering. On the eight datasets on which all methods completed, HyperFuse needed 8.7\,s per dataset on average, a geometric-mean speed-up of 13--179$\times$ over the baselines. Its accuracy averaged over the datasets was the highest with five of the six classifiers, and its classification and clustering scores did not differ significantly from those of TriCL and SE-HSSL. Compared with HypeBoy, the fastest baseline, HyperFuse was 13$\times$ faster and 2.1--4.1 percentage points more accurate with every classifier. With one configuration for all datasets and seconds per embedding, HyperFuse suits settings that require many embeddings.}


\maketitle

\section{Introduction}\label{sec:intro}

Many relational systems involve interactions among more than two entities: a paper has several authors, a protein complex contains several proteins, and a set of documents shares a keyword. Hypergraphs represent such higher-order interactions directly, with hyperedges that connect arbitrary subsets of nodes \citep{battiston2020networks,klamt2009hypergraphs}. A common way to make hypergraph data usable by standard learning algorithms is to map each node to a low-dimensional vector, an \emph{embedding}, that summarises its attributes and its position in the hypergraph \citep{antelmi2023survey}. Because labels are often scarce, self-supervised methods, which learn embeddings without labels and are then evaluated on downstream tasks, have become the dominant approach \citep{kim2024survey}.

Current self-supervised hypergraph methods such as TriCL \citep{lee2023tricl}, SE-HSSL \citep{li2024sehssl}, VilLain \citep{lee2024villain} and HypeBoy \citep{kim2024hypeboy} produce embeddings that support accurate node classification and clustering. They do so by training neural encoders with contrastive, correlation-based or generative objectives for hundreds to thousands of epochs, frequently with objectives that compare every node with every other node. In our benchmark (Sect.~\ref{sec:results}), the official configurations of these methods required between 4\,s and 9.1\,h per hypergraph, and generating one embedding of PubMed (3{,}840 nodes) took 9.1\,h with TriCL, 3.5\,h with SE-HSSL and 1.4\,h with VilLain on our hardware; HypeBoy, the fastest baseline, needed 1.0\,min on PubMed but 2.3\,h on 20News.

Such costs are acceptable when a single, fixed hypergraph is embedded once. They are less acceptable in two settings that motivate this work. First, in \emph{collections of hypergraphs}, every member needs its own embedding and the cost is multiplied by the size of the collection; sample-specific molecular networks, in which each patient is represented by an individual interaction network, are one example \citep{kuijjer2019lioness}. Second, in \emph{evolving hypergraphs}, node attributes and memberships change over time and embeddings must be recomputed for new snapshots \citep{kazemi2020dynamic}. In both settings, the time needed to produce one embedding determines whether an analysis is feasible.

We introduce \emph{HyperFuse}, a self-supervised embedding method for attributed hypergraphs that prioritises embedding-generation time while aiming to retain the quality of current self-supervised methods. HyperFuse builds on FUSE \citep{chakraborty2026fuse}, which showed that a scalable approximation of modularity maximisation provides structural node embeddings at low cost for ordinary graphs in a semi-supervised setting. HyperFuse is fully unsupervised and extends this idea to hypergraphs with node attributes. It proceeds in three stages. Stage~A computes structural coordinates by maximising a spectral relaxation of modularity defined on the hypergraph adjacency matrix of \citet{banerjee2021spectrum}; the modularity operator is applied without forming any $N\times N$ matrix. Stage~B computes feature summaries that are consistent under two corruptions of the hypergraph and estimates a bounded utility weight for every hyperedge. Stage~C trains a small hypergraph encoder whose aggregation is weighted by these utilities, using a correlation-based self-supervised objective \citep{zhang2021ccassg} for a fixed budget of 100 epochs.

Our contributions are as follows.
\begin{itemize}
\item We describe HyperFuse, a three-stage, label-free pipeline for node embedding generation in attributed hypergraphs, in which inexpensive linear-algebraic stages (A and B) supply structural, spectral and hyperedge-reliability information to a short self-supervised training stage (C).
\item We build hypergraph modularity on Banerjee's adjacency matrix, whose strengths equal the node degrees \citep{banerjee2021spectrum}, obtain a modularity operator that is applied in time linear in the number of incidences without forming any $N\times N$ matrix, and state conditions under which the Stage~A iteration converges to the leading eigenspace of the modularity matrix (Propositions~\ref{prop:strength}--\ref{prop:ascent}).
\item We report a benchmark on nine public hypergraphs against four self-supervised baselines run with their official configurations, evaluating embeddings with six downstream classifiers, two clustering scores and wall-clock embedding time, together with rank-based statistical tests.
\end{itemize}

On the eight datasets on which all methods completed, HyperFuse was the fastest method on every dataset, with geometric-mean speed-ups of 143$\times$ over TriCL, 96$\times$ over SE-HSSL, 179$\times$ over VilLain and 13$\times$ over HypeBoy, its accuracy averaged over the datasets was the highest with five of the six downstream classifiers, and its classification and clustering scores did not differ significantly from those of TriCL and SE-HSSL and were higher than those of HypeBoy with every classifier. We discuss these results and the limitations of the study in Sect.~\ref{sec:discussion}.

\section{Related work}\label{sec:related}

\paragraph{Hypergraph neural networks.}
Spectral hypergraph learning dates back at least to the normalised hypergraph Laplacian of \citet{zhou2006learning}. Hypergraph neural networks generalise graph convolution to hypergraphs through node-to-hyperedge and hyperedge-to-node aggregation: HGNN \citep{feng2019hgnn} uses the normalised incidence operator, HyperGCN \citep{yadati2019hypergcn} approximates hyperedges by weighted pairwise edges, HNHN \citep{dong2020hnhn} introduces hyperedge neurons with separate normalisations, UniGNN \citep{huang2021unignn} transfers graph architectures to hypergraphs, and AllSet \citep{chien2022allset} formulates both aggregation steps as learnable multiset functions. We use these models as downstream classifiers rather than as competitors. Surveys are given by \citet{antelmi2023survey} and \citet{kim2024survey}.

\paragraph{Self-supervised hypergraph learning.}
HyperGCL \citep{wei2022hypergcl} studies fabricated and generative augmentations for hypergraph contrastive learning. TriCL \citep{lee2023tricl} contrasts two augmented views at the node, hyperedge and node--hyperedge membership levels. SE-HSSL \citep{li2024sehssl} replaces the node- and hyperedge-level contrast by sampling-free objectives based on canonical correlation analysis and adds a hierarchical membership-level contrast. HypeBoy \citep{kim2024hypeboy} is a generative approach based on hyperedge filling. VilLain \citep{lee2024villain} learns embeddings from the structure alone by propagating learned virtual labels. These methods are closest to our setting; we compare with TriCL, SE-HSSL, VilLain and HypeBoy, which cover contrastive, correlation-based, structure-only and generative learning. HyperGCL is not included because running every baseline with its official configuration on all datasets and seeds already required about 111 GPU-hours. On graphs, related objectives include mutual-information maximisation \citep{velickovic2019dgi}, node-level contrast \citep{zhu2020grace}, bootstrapping \citep{thakoor2022bgrl}, and the invariance-plus-decorrelation objective of CCA-SSG \citep{zhang2021ccassg}, which is related to Barlow Twins \citep{zbontar2021barlow}. Stage~C of HyperFuse uses the CCA-SSG objective.

\paragraph{Efficient embeddings and modularity.}
Random-walk methods such as DeepWalk \citep{perozzi2014deepwalk} and node2vec \citep{grover2016node2vec} are unsupervised but require sampling and training; NetMF \citep{qiu2018netmf} and ProNE \citep{zhang2019prone} obtain embeddings by matrix factorisation and spectral propagation at lower cost. Decoupling fixed propagation from learning is also the basis of SGC \citep{wu2019sgc} and SSGC \citep{zhu2021ssgc}, whose multi-hop averaging we adopt. Modularity \citep{newman2006modularity} measures the excess of within-group connections over a degree-preserving null model, and hypergraph variants have been proposed for clustering \citep{kumar2020hypergraph,chodrow2021generative}. FUSE \citep{chakraborty2026fuse} maximises a relaxed graph modularity by projected gradient ascent with QR re-orthonormalisation and combines it with label-aware terms. HyperFuse retains the modularity ascent but with the exact modularity matrix instead of FUSE's surrogate gradient, replaces the graph adjacency by Banerjee's hypergraph adjacency \citep{banerjee2021spectrum}, removes all label information and adds feature and self-supervised stages.

\section{Problem setting and notation}\label{sec:prelim}

Scalars are set in italic ($N$, $d$), vectors in bold lower case ($\vect{u}$), matrices in bold upper case ($\mat{H}$) and sets in calligraphic type ($\mathcal{V}$). The $i$-th row of a matrix $\mat{Y}$ is $\vect{y}_i$, $\one_n$ and $\mat{I}_n$ are the all-ones vector and the identity of size $n$, $\diag(\vect{v})$ is the diagonal matrix with diagonal $\vect{v}$, $\|\cdot\|_F$ is the Frobenius norm, $[\mat{Y}\,\|\,\mat{Y}']$ is column-wise concatenation, and $\qf(\mat{Y})$ is the orthonormal factor of the thin QR decomposition of $\mat{Y}$. For two matrices with the same number of rows, $\cosr_i(\mat{Y},\mat{Y}')=\vect{y}_i^\top\vect{y}'_i/(\|\vect{y}_i\|\,\|\vect{y}'_i\|)$ is the cosine similarity of their $i$-th rows. Table~\ref{tab:notation} lists the main symbols.

An attributed hypergraph is a triple $(\mathcal{V},\mathcal{E},\mat{X})$ with node set $\mathcal{V}=\{v_1,\dots,v_N\}$, hyperedge set $\mathcal{E}=\{e_1,\dots,e_M\}$, $e_j\subseteq\mathcal{V}$, and node features $\mat{X}\in\R^{N\times D}$. HyperFuse discards hyperedges with fewer than two members; $M$ denotes the number of remaining hyperedges. The incidence matrix $\mat{H}\in\{0,1\}^{N\times M}$ has $H_{ij}=1$ if $v_i\in e_j$. The degree of $v_i$ is $\deg(v_i)=\sum_j H_{ij}$, the size of $e_j$ is $|e_j|=\sum_i H_{ij}$, and $\mat{D}_v=\diag(\mat{H}\one_M)$ and $\mat{D}_e=\diag(\mat{H}^\top\one_N)$. The number of incidences is $\nnz(\mat{H})=\sum_j|e_j|$. The task is to compute, without labels, an embedding matrix $\mat{Z}\in\R^{N\times d}$ that is useful for downstream tasks, here node classification and clustering, and to do so quickly.

\begin{table}[t]
\caption{Main notation}\label{tab:notation}
\begin{tabular}{@{}ll@{}}
\toprule
Symbol & Meaning \\
\midrule
$N$, $M$, $D$, $d$ & numbers of nodes, hyperedges, input features; embedding width \\
$d_s$, $d_f$, $d_g$ & widths of the structural, feature and incidence summaries ($d/4$, $d/2$, $d/4$) \\
$\mat{H}$, $\mat{D}_v$, $\mat{D}_e$ & incidence matrix, node-degree and hyperedge-size matrices \\
$\mat{A}$, $\vect{k}$, $\mat{B}$ & Banerjee adjacency, strength vector, modularity matrix \\
$\mat{S}$, $\mat{F}$, $\mat{G}$ & structural, feature and incidence summaries (Stages A, B) \\
$\mat{R}$, $\mathcal{T}_{\mat{R}}$, $\boldsymbol{\Phi}_\ell$ & normalised incidence, lazy propagation, propagation at scale $\ell$ \\
$\vect{u}\in\R^M$ & hyperedge utility weights \\
$\mathcal{A}^{\boldsymbol\mu}$ & utility-weighted node--hyperedge--node aggregation under keep-mask $\boldsymbol\mu$ \\
$f_{\boldsymbol\theta}$, $\mat{Z}',\mat{Z}''$, $\mat{Z}$ & encoder, embeddings of two views, final embedding \\
\botrule
\end{tabular}
\end{table}

\section{HyperFuse}\label{sec:method}

Figure~\ref{fig:arch} summarises the pipeline and Algorithm~\ref{alg:hyperfuse} gives it in pseudocode. Stages~A and B involve only linear algebra, truncated decompositions and a small ridge regression; Stage~C is the only stage with a neural network trained by gradient descent. No stage uses labels.

\begin{figure}[t]
\centering
\resizebox{\textwidth}{!}{%
\begin{tikzpicture}[
    block/.style={rectangle, draw=black!70, fill=blue!5, thick, rounded corners=3pt, minimum width=4.4cm, minimum height=1.6cm, align=center, font=\sffamily\small, drop shadow={opacity=0.15}},
    stagebox/.style={rectangle, draw=gray!60, fill=gray!5, dashed, rounded corners=6pt, inner sep=8pt},
    arrow/.style={-{Stealth[length=2.4mm]}, thick, color=black!70},
    data/.style={ellipse, draw=orange!70!black, fill=orange!10, thick, minimum width=2.4cm, minimum height=1.0cm, align=center, font=\sffamily\footnotesize},
    lab/.style={font=\sffamily\footnotesize, fill=white, inner sep=1pt}
]
    \node[data] (H) at (0, 1.9) {Incidence\\$\mat{H}\in\{0,1\}^{N\times M}$};
    \node[data] (X) at (0,-1.9) {Node features\\$\mat{X}\in\R^{N\times D}$};
    \node[block] (A) at (5.2, 1.9) {Banerjee adjacency $\mat{A}$,\\modularity matrix $\mat{B}$ (matrix-free)\\projected ascent $\Rightarrow \mat{S}\in\R^{N\times d_s}$};
    \node[block] (B) at (5.2,-1.9) {residualisation w.r.t.\ $\mat{S}$,\\two corrupted views, three scales\\truncated SVD $\Rightarrow \mat{F},\mat{G}$\\stability regression $\Rightarrow \vect{u}\in\R^{M}$};
    \node[block] (C) at (12.6, 0) {two corrupted views per epoch\\utility-weighted hypergraph encoder\\invariance + decorrelation loss};
    \node[data] (Z) at (17.0, 0) {Embeddings\\$\mat{Z}\in\R^{N\times d}$};
    \draw[arrow] (H) -- (A);
    \draw[arrow] (X) -- (B);
    \draw[arrow] (H.south east) -- node[lab, pos=0.35] {$\mat{H}$} (B.north west);
    \draw[arrow] (A) -- node[lab, right] {$\mat{S}$} (B);
    \draw[arrow] (B.east) -- ++(1.0,0) |- node[lab, pos=0.78, above=2pt] {$[\mat{F}\,\|\,\mat{G}]$, $\vect{u}$} ($(C.west)+(0,-0.3)$);
    \draw[arrow] (H.north) -- ++(0,1.7) -| node[lab, pos=0.25] {$\mat{H}$} ($(C.north)+(1.5,0)$);
    \draw[arrow] (X.south) -- ++(0,-1.7) -| node[lab, pos=0.25] {$\mat{X}$ (row-normalised)} ($(C.south)+(1.5,0)$);
    \draw[arrow] (C) -- (Z);
    \begin{scope}[on background layer]
        \node[stagebox, fit=(A), label={[font=\sffamily\bfseries\color{blue!40!black}]north:Stage A: structure}] {};
        \node[stagebox, fit=(B), label={[font=\sffamily\bfseries\color{blue!40!black}]south:Stage B: feature summaries and hyperedge utility}] {};
        \node[stagebox, fit=(C), label={[font=\sffamily\bfseries\color{blue!40!black}, xshift=-1.5cm]north:Stage C: self-supervised encoder}] {};
    \end{scope}
\end{tikzpicture}}
\caption{Overview of HyperFuse. Stage~A computes structural coordinates $\mat{S}$ by maximising a spectral relaxation of hypergraph modularity. Stage~B computes feature summaries $\mat{F}$ and incidence summaries $\mat{G}$, and hyperedge utility weights $\vect{u}$ derived from the stability of hyperedge members under feature and membership masking. Stage~C trains a two-layer hypergraph encoder, whose hyperedge-to-node aggregation is weighted by $\vect{u}$, on two corrupted views with an invariance-plus-decorrelation objective. No stage uses labels.}\label{fig:arch}
\end{figure}
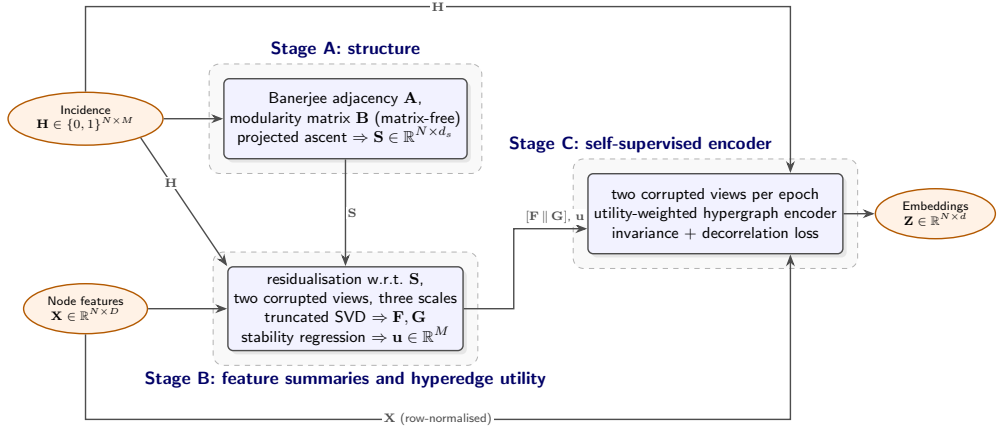

\subsection{Stage A: modularity on Banerjee's hypergraph adjacency}\label{sec:stageA}

Let $\mat{W}=\diag\big(1/(|e_1|-1),\dots,1/(|e_M|-1)\big)$. Following \citet{banerjee2021spectrum}, the adjacency matrix of the hypergraph is
\begin{equation}\label{eq:adj}
\mat{A}=\mat{H}\mat{W}\mat{H}^\top-\diag(\mat{H}\mat{W}\one_M),\qquad
A_{il}=\sum_{j:\,v_i,v_l\in e_j}\frac{1}{|e_j|-1}\quad(i\neq l),\qquad A_{ii}=0 .
\end{equation}
Each hyperedge thus distributes a total weight of one from each of its members to the other members. With strengths $\vect{k}=\mat{A}\one_N$ and $2m=\one_N^\top\vect{k}$, the modularity matrix \citep{newman2006modularity} is
\begin{equation}\label{eq:mod}
\mat{B}=\mat{A}-\frac{1}{2m}\vect{k}\vect{k}^\top .
\end{equation}

\begin{proposition}[Strengths]\label{prop:strength}
For the adjacency~\eqref{eq:adj}, $\vect{k}=\mat{H}\one_M$, that is, $k_i=\deg(v_i)$, and $2m=\nnz(\mat{H})$.
\end{proposition}
\begin{proof}
For each $i$, $k_i=\sum_{l\neq i}A_{il}=\sum_{j:\,v_i\in e_j}\sum_{l\in e_j\setminus\{i\}}(|e_j|-1)^{-1}=\sum_{j:\,v_i\in e_j}1=\deg(v_i)$. Summing over $i$ gives $2m=\sum_i\deg(v_i)=\sum_j|e_j|=\nnz(\mat{H})$.
\end{proof}

The identity $\vect{k}=\mat{H}\one_M$ is noted by \citet[Sect.~2]{banerjee2021spectrum}; we state it because the modularity null model relies on it. Proposition~\ref{prop:strength} implies that the null model in~\eqref{eq:mod} depends only on the node strengths, which in this construction equal the number of hyperedges incident to each node. It helps us to avoid the explicit calculation of $\mat{A}$ for obtaining node-strength vector and total edge weight. Materialising $\mat{A}$ can be computationally expensive: a hyperedge with
$|e|$ members contributes up to $|e|(|e|-1)$ non-zeros to $\mat{A}$
but only $|e|$ incidences to $\mat{H}$. On 20News, $\mat{A}$ would have
$53.3\times10^6$ non-zeros compared with $6.5\times10^4$ incidences
(Table~\ref{tab:datasets}).

Stage~A computes $\mat{S}\in\R^{N\times d_s}$ with orthonormal columns
that maximises $\tr(\mat{S}^\top\mat{B}\mat{S})$, the standard spectral
relaxation of modularity maximisation into $d_s$ soft groups. We use
the projected gradient ascent with QR re-orthonormalisation of FUSE
\citep{chakraborty2026fuse}, applied to the exact modularity matrix
$\mat{B}$ (FUSE uses a surrogate gradient that replaces
$\vect{k}^\top\mat{S}$ by $\one_N^\top\mat{S}$),
\begin{equation}\label{eq:ascent}
\mat{S}^{(t+1)}
=
\qf\big(\mat{S}^{(t)}+\eta\,\mat{B}\mat{S}^{(t)}\big),
\qquad
\mat{S}^{(0)}=\qf(\boldsymbol\Omega),
\quad
\Omega_{ik}\overset{\text{i.i.d.}}{\sim}\mathcal{N}(0,1),
\end{equation}
with step size
$\eta=2c/(\lambda_1+|\min(\lambda_N,0)|)$, $c=0.9$, where
$\lambda_1\ge\dots\ge\lambda_N$ are the eigenvalues of $\mat{B}$. Crucially, Proposition~\ref{prop:matfree} allows the product
$\mat{B}\mat{S}^{(t)}$ required at each iteration to be evaluated
directly from the sparse incidence matrix $\mat{H}$, without explicitly
materialising either $\mat{A}$ or $\mat{B}$.

\begin{proposition}[Calculating $\mat{B}\mat{S}$ without calculating $\mat{A}$ or $\mat{B}$]
\label{prop:matfree}
For the Stage-A iterate $\mat{S}\in\R^{N\times d_s}$,
\begin{equation}
\label{eq:matfree}
\mat{B}\mat{S}
=
\mat{H}\bigl(\mat{W}(\mat{H}^\top\mat{S})\bigr)
-
\diag(\mat{H}\mat{W}\one_M)\,\mat{S}
-
\frac{1}{\nnz(\mat{H})}\,
\vect{k}\,(\vect{k}^\top\mat{S}).
\end{equation}
Thus, $\mat{B}\mat{S}$ can be computed directly from the sparse
incidence matrix $\mat{H}$ without forming either $\mat{A}$ or
$\mat{B}$. The multiplication requires
\[
O\bigl((\nnz(\mat{H})+N)d_s\bigr)
\]
operations and
\[
O\bigl(\nnz(\mat{H})+(N+M)d_s\bigr)
\]
memory.
\end{proposition}

\begin{proof}
Substituting~\eqref{eq:adj} into~\eqref{eq:mod} and using
Proposition~\ref{prop:strength} gives~\eqref{eq:matfree}. The sparse
products $\mat{H}^\top\mat{S}$ and
$\mat{H}(\mat{W}(\mat{H}^\top\mat{S}))$ each require
$O(\nnz(\mat{H})d_s)$ operations. The diagonal and rank-one terms
require $O(Nd_s)$ operations, while applying the diagonal matrix
$\mat{W}$ costs $O(Md_s)$. Since
$M\leq\nnz(\mat{H})$, the total cost is
$O((\nnz(\mat{H})+N)d_s)$. No $N\times N$ matrix is formed.
\end{proof}

\begin{proposition}[Convergence of the projected ascent]
\label{prop:ascent}
Let
\[
\lambda_1\geq\lambda_2\geq\cdots\geq\lambda_N
\]
be the eigenvalues of the symmetric modularity matrix $\mat{B}$, and
suppose that
\[
\lambda_{d_s}\geq 0,
\qquad
\lambda_{d_s}>\lambda_{d_s+1}.
\]
For any $c\in(0,1)$, choose the step size
\begin{equation}
\eta
=
\frac{2c}
{\lambda_1+|\min(\lambda_N,0)|}.
\label{eq:eta_proof}
\end{equation}
Consider the projected iteration
\begin{equation}
\mat{S}^{(t+1)}
=
\qf\!\left(
\mat{S}^{(t)}+\eta\mat{B}\mat{S}^{(t)}
\right),
\qquad
\mat{S}^{(0)}=\qf(\boldsymbol{\Omega}),
\label{eq:prop_iteration}
\end{equation}
where $\boldsymbol{\Omega}$ has independent standard Gaussian entries.
Then, with probability one over $\boldsymbol{\Omega}$,
\[
\operatorname{span}(\mat{S}^{(t)})
\longrightarrow
\operatorname{span}
\{\vect{u}_1,\ldots,\vect{u}_{d_s}\},
\]
where $\vect{u}_i$ is an eigenvector of $\mat{B}$ associated with
$\lambda_i$. This limiting subspace is the unique solution of
\[
\max_{\mat{S}^{\top}\mat{S}=\mat{I}_{d_s}}
\tr(\mat{S}^{\top}\mat{B}\mat{S}).
\]
Moreover, the convergence is geometric, with asymptotic ratio
\begin{equation}
\rho
=
\max_{i>d_s}
\frac{|1+\eta\lambda_i|}
{1+\eta\lambda_{d_s}}
<1.
\label{eq:rho_proof}
\end{equation}
\end{proposition}

\begin{proof}
The proof has two parts. First, we identify the subspace that maximizes
the relaxed modularity objective. Second, we show that the proposed
iteration converges to this subspace.

Because $\mat{B}$ is symmetric, it admits the eigendecomposition
\[
\mat{B}
=
\mat{U}\boldsymbol{\Lambda}\mat{U}^{\top},
\qquad
\boldsymbol{\Lambda}
=
\diag(\lambda_1,\ldots,\lambda_N),
\]
with
\[
\lambda_1\geq\lambda_2\geq\cdots\geq\lambda_N.
\]
By the Ky Fan maximum principle \citep{fan1949theorem},
\[
\max_{\mat{S}^{\top}\mat{S}=\mat{I}_{d_s}}
\tr(\mat{S}^{\top}\mat{B}\mat{S})
=
\sum_{i=1}^{d_s}\lambda_i.
\]
Hence the maximizing subspace is the span of the eigenvectors
corresponding to the largest $d_s$ eigenvalues. Because
$\lambda_{d_s}>\lambda_{d_s+1}$, this subspace is unique:
\[
\mathcal{U}_*
=
\operatorname{span}
\{\vect{u}_1,\ldots,\vect{u}_{d_s}\}.
\]
It remains to show that the iteration in~\eqref{eq:prop_iteration}
converges to $\mathcal{U}_*$.

Define the shifted matrix
\[
\mat{C}
=
\mat{I}_N+\eta\mat{B}.
\]
Then~\eqref{eq:prop_iteration} can be written as
\[
\mat{S}^{(t+1)}
=
\qf\!\left(\mat{C}\mat{S}^{(t)}\right).
\]
Since $\mat{C}$ is an affine transformation of $\mat{B}$, it has the
same eigenvectors, with eigenvalues
\[
\mu_i=1+\eta\lambda_i,
\qquad i=1,\ldots,N.
\]

This shift is important because orthogonal iteration is governed by
the eigenvalues of the iteration matrix in \emph{magnitude}, whereas
the relaxed modularity objective is governed by the eigenvalues of
$\mat{B}$ in algebraic order. Thus, without controlling the negative
part of the spectrum, a sufficiently negative eigenvalue of $\mat{B}$
could produce an eigenvalue of $\mat{C}$ with large magnitude and
become dominant in the iteration even though it does not maximize the
modularity objective. The choice of $\eta$ in~\eqref{eq:eta_proof}
prevents this from happening.

For every $i\leq d_s$, the assumption $\lambda_{d_s}\geq0$ gives
\[
\mu_i
=
1+\eta\lambda_i
\geq1.
\]
In particular,
\[
\mu_{d_s}
=
1+\eta\lambda_{d_s}
\geq1.
\]

Now consider any negative eigenvalue $\lambda_i<0$. If
$\lambda_N<0$, then
\[
\eta|\lambda_i|
\leq
\eta|\lambda_N|
=
\frac{2c|\lambda_N|}
{\lambda_1+|\lambda_N|}
<2c<2.
\]
If $\lambda_N\geq0$, there are no negative eigenvalues and this case
does not arise. Therefore, whenever $\lambda_i<0$,
\[
-1
<
1-\eta|\lambda_i|
<
1,
\]
and hence
\begin{equation}
|1+\eta\lambda_i|<1.
\label{eq:negative_inside}
\end{equation}
Thus every negative eigenvalue of $\mat{B}$ is mapped strictly inside
the unit circle by the shift.

For the remaining non-negative eigenvalues with $i>d_s$, the ordering
is preserved because $\eta>0$. Hence
\[
1+\eta\lambda_i
\leq
1+\eta\lambda_{d_s+1}
<
1+\eta\lambda_{d_s}.
\]
Combining this with~\eqref{eq:negative_inside} and
$\mu_{d_s}\geq1$ gives
\[
\max_{i>d_s}|1+\eta\lambda_i|
<
1+\eta\lambda_{d_s}.
\]
Consequently,
\[
\rho
=
\max_{i>d_s}
\frac{|1+\eta\lambda_i|}
{1+\eta\lambda_{d_s}}
<1.
\]

It follows that
\[
\mu_1,\ldots,\mu_{d_s}
\]
are precisely the $d_s$ eigenvalues of $\mat{C}$ having largest
magnitude. Thus, the eigenspace that maximizes the relaxed modularity
objective is also the dominant invariant subspace of the shifted
matrix $\mat{C}$.

The update in~\eqref{eq:prop_iteration} is therefore an orthogonal
iteration for $\mat{C}$. The QR factor $\qf(\cdot)$ only
re-orthonormalises the columns and does not change their span. Standard
convergence theory for orthogonal iteration
\citep[Sect.~8.2]{golub2013matrix} therefore gives geometric
convergence to the dominant invariant subspace, with asymptotic ratio
$\rho$ in~\eqref{eq:rho_proof}.

Finally, the Gaussian initialization satisfies the usual
non-degeneracy condition for orthogonal iteration with probability one.
Indeed, the projection of $\boldsymbol{\Omega}$ onto any fixed
$d_s$-dimensional subspace has full column rank almost surely, because
rank deficiency corresponds to a measure-zero subset of the continuous
Gaussian sample space. Hence the initial subspace has a nonzero
component in every direction of the dominant eigenspace almost surely.

Therefore,
\[
\operatorname{span}(\mat{S}^{(t)})
\longrightarrow
\mathcal{U}_*
=
\operatorname{span}
\{\vect{u}_1,\ldots,\vect{u}_{d_s}\}
\qquad\text{with probability one}.
\]
Since $\mathcal{U}_*$ is the unique maximizing subspace of the relaxed
modularity objective by the Ky Fan maximum principle
\citep{fan1949theorem}, the proposition follows.
\end{proof}
\begin{remark}\label{rem:impl}
The implementation computes $\lambda_1$ and $\lambda_N$ with a Lanczos solver when $N\le20{,}000$ and otherwise (or when the Lanczos solver fails) uses a power-iteration estimate $\hat\rho$ of the spectral radius, setting $\lambda_1=-\lambda_N=\hat\rho$ (so $\eta=c/\hat\rho$). It stops when the relative change of $\tr(\mat{S}^\top\mat{B}\mat{S})$ stays below $10^{-5}$ for five consecutive checks (every five iterations, at least 20 iterations) or after 1{,}000 iterations. $\mat{S}$ is therefore an approximation of the leading eigenspace. In the power-iteration branch, the condition $\eta|\lambda_N|<2$ used in the proof of Proposition~\ref{prop:ascent} requires $\hat\rho>c|\lambda_N|/2$; on DBLP, the only dataset in this branch, $\eta|\lambda_N|=0.86$. The assumptions of Proposition~\ref{prop:ascent} hold on eight of the nine datasets; on Zoo, $\mat{B}$ has only 12 positive eigenvalues and $\lambda_{32}=\lambda_{33}$, so the leading 32-dimensional eigenspace is not unique. On the five datasets we checked by exact eigendecomposition (Zoo, Cora-CC, Citeseer, Cora-CA and NTU2012), $\rho$ ranged from 0.9965 to 0.9999 where defined, and after early stopping (235--430 iterations) $\tr(\mat{S}^\top\mat{B}\mat{S})$ was within 0.1\% of its maximum, although individual eigenvectors close to $\lambda_{d_s}$ were not resolved. Stage~B uses $\mat{S}$ mainly through the projector $\mat{S}\mat{S}^\top$ and pairwise cosine similarities, which depend only on the subspace.
\end{remark}

\subsection{Stage B: feature summaries and hyperedge utilities}\label{sec:stageB}

\paragraph{Propagation.}
Let $\mat{R}=\mat{D}_v^{-1/2}\mat{H}\mat{D}_e^{-1/2}$ be the normalised incidence and, for any feature matrix $\mat{Y}$, let $\mathcal{T}_{\mat{R}}(\mat{Y})=\alpha\mat{Y}+(1-\alpha)\mat{R}\mat{R}^\top\mat{Y}$ be a lazy propagation step with $\alpha\in[0,1]$. We use three propagation scales,
\begin{equation}\label{eq:scales}
\boldsymbol\Phi_0(\mat{Y};\mat{R})=\mat{Y},\qquad
\boldsymbol\Phi_1(\mat{Y};\mat{R})=\mathcal{T}_{\mat{R}}(\mat{Y}),\qquad
\boldsymbol\Phi_2(\mat{Y};\mat{R})=\frac1K\sum_{k=1}^{K}\mathcal{T}_{\mat{R}}^{k}(\mat{Y}),
\end{equation}
where the third scale is the multi-hop average of SSGC \citep{zhu2021ssgc}.

\paragraph{Corruptions.}
Two corruptions are used throughout. Feature masking $\FM_\rho(\mat{Y})$ sets each entry of $\mat{Y}$ to zero independently with probability $\rho$ and rescales the remaining entries by $1/(1-\rho)$. Membership masking $\MM_\rho(\mat{H})$ removes $\lfloor\rho|e_j|\rfloor$ uniformly chosen members from each hyperedge $e_j$ while keeping at least two members. We write $\mat{H}^{\mathrm{b}}=\MM_\rho(\mat{H})$ and $\mat{R}^{\mathrm{b}}$ for its normalised incidence.

\paragraph{Feature summary $\mat{F}$.}
Let $\bar{\mat{X}}$ be $\mat{X}$ with rows scaled to unit $\ell_2$ norm and columns centred. To reduce redundancy with Stage~A, we partially remove the component of the features that lies in the span of $\mat{S}$,
\begin{equation}
\tilde{\mat{X}}=\bar{\mat{X}}-\beta\,\mat{S}\mat{S}^\top\bar{\mat{X}},\qquad\beta\in[0,1].
\end{equation}
We form two views, one with masked features on the original structure and one with the original features on masked structure,
\[
\mat{Y}^{\mathrm{a}}_\ell=\boldsymbol\Phi_\ell\big(\FM_\rho(\tilde{\mat{X}});\mat{R}\big),\qquad
\mat{Y}^{\mathrm{b}}_\ell=\boldsymbol\Phi_\ell\big(\tilde{\mat{X}};\mat{R}^{\mathrm{b}}\big),\qquad \ell\in\{0,1,2\},
\]
and measure, for every node and scale, how well the two views agree, $c_{i\ell}=\cosr_i(\mat{Y}^{\mathrm{a}}_\ell,\mat{Y}^{\mathrm{b}}_\ell)$. The agreements define per-node scale weights
\begin{equation}\label{eq:scaleweights}
\pi_{i\ell}=(1-3\omega)\,\frac{\exp(c_{i\ell}/\tau)}{\sum_{\ell'=0}^{2}\exp(c_{i\ell'}/\tau)}+\omega,
\end{equation}
with temperature $\tau>0$ and floor $\omega\in[0,1/3]$, so that each node emphasises the scales at which its representation is least sensitive to corruption. With $\boldsymbol\Pi_\ell=\diag(\pi_{1\ell},\dots,\pi_{N\ell})$, let $\mat{Y}^{\mathrm{cat}}=[\boldsymbol\Pi_0\mat{Y}^{\mathrm{a}}_0\,\|\,\boldsymbol\Pi_1\mat{Y}^{\mathrm{a}}_1\,\|\,\boldsymbol\Pi_2\mat{Y}^{\mathrm{a}}_2]\in\R^{N\times3D}$. The feature summary is its rank-$d_f$ truncated singular value decomposition, $\mat{F}=\gamma_F\,\mat{U}_{d_f}\boldsymbol\Sigma_{d_f}$, where $\gamma_F$ rescales $\mat{F}$ to $\|\mat{F}\|_F^2=d_f$ and each column sign is fixed so that its entry of largest magnitude is positive. If $\mat{Y}^{\mathrm{cat}}$ has fewer than $d_f+1$ columns (Zoo), $\mat{F}$ is narrower; it is never zero-padded.

\paragraph{Incidence summary $\mat{G}$.}
$\mat{G}=\gamma_G\,\mat{U}^{\mat{R}}_{d_g}\boldsymbol\Sigma^{\mat{R}}_{d_g}$ consists of the $d_g$ leading left singular vectors of $\mat{R}$ scaled by their singular values, with the same rescaling and sign convention. It is computed by a dense SVD when $\min(N,M)\le3{,}000$ and otherwise by LOBPCG \citep{knyazev2001lobpcg}, with randomised subspace iteration \citep{halko2011finding} as a fallback when the LOBPCG residual exceeds $10^{-2}$.

\paragraph{Hyperedge utility $\vect{u}$.}
Hyperedges differ in how informative they are; some group nodes with dissimilar attributes or are dominated by high-degree nodes. HyperFuse assigns every hyperedge a weight that reflects how stable the representations of its members are under the two corruptions. For each node we compute, from the unresidualised features,
\begin{align*}
\sigma_i&=\frac14\Big(\frac13\sum_{\ell=0}^{2}\cosr_i\big(\boldsymbol\Phi_\ell(\bar{\mat{X}};\mat{R}),\boldsymbol\Phi_\ell(\FM_\rho(\bar{\mat{X}});\mat{R})\big)+\frac13\sum_{\ell=0}^{2}\cosr_i\big(\boldsymbol\Phi_\ell(\bar{\mat{X}};\mat{R}),\boldsymbol\Phi_\ell(\bar{\mat{X}};\mat{R}^{\mathrm{b}})\big)\Big),\\
\kappa_i&=\frac13\sum_{\ell=0}^{2}\cosr_i\big(\boldsymbol\Phi_\ell(\FM_\rho(\bar{\mat{X}});\mat{R}),\boldsymbol\Phi_\ell(\bar{\mat{X}};\mat{R}^{\mathrm{b}})\big),
\end{align*}
the stability of $v_i$ with respect to each corruption and the agreement between the two corrupted views. For hyperedge $e_j$, with $\operatorname{mean}_j$ and $\operatorname{sd}_j$ the mean and standard deviation over its members,
\[
r_j=0.65\operatorname{mean}_j(\sigma)+0.35\operatorname{mean}_j(\kappa)-0.20\operatorname{sd}_j(\sigma),
\]
and $r_j^{0}$ is the value of $0.65\operatorname{mean}(\sigma)+0.35\operatorname{mean}(\kappa)$ averaged over three uniformly random node sets of size $|e_j|$. The stability target $t_j\in[0,1]$ is the normalised rank of $r_j-r_j^{0}$ among all hyperedges, so that hyperedges are compared with size-matched random groups. The coefficients $0.65$, $0.35$, $0.20$ and the number of random sets are fixed design constants.

Stability targets from a single corruption draw are noisy. We therefore regress them on 13 descriptors $\boldsymbol\psi_j\in\R^{13}$ of each hyperedge: the mean pairwise cosine similarity of its members in $\mat{S}$, $\mat{F}$ and $\mat{G}$ (three descriptors); the mean, minimum and standard deviation of its members' scale agreements $c_{i\ell}$ (three); $|e_j|$ and $\log(1+|e_j|)$ (two); the mean member degree and its coefficient of variation (two); and the mean absolute entry of the normalised member centroid in $\mat{S}$, $\mat{F}$ and $\mat{G}$ (three). A ridge regression on standardised descriptors, fitted by 160 gradient steps (step size $0.03$, penalty $0.05$), gives predictions $\hat t_j$ clipped to $[0,1]$; this projects the targets onto the part explained by the descriptors. With $z_j$ the standardised predictions, the utilities are
\begin{equation}\label{eq:utility}
\tilde u_j=\exp\big(\operatorname{clip}(0.8\,z_j,-2,2)\big),\qquad
\vect{u}=\Pi_{[u_{\min},u_{\max}]}^{\,\mathrm{mean}=1}(\tilde{\vect{u}}),
\end{equation}
where $\Pi$ rescales $\tilde{\vect{u}}$ to mean one and then maps it to weights in $[u_{\min},u_{\max}]=[0.65,1.35]$ with mean one by alternating clipping and redistribution of the excess over the unclipped weights. The bounds keep every hyperedge in the model and limit the influence of the utility estimate.

\subsection{Stage C: utility-weighted correlation-based self-supervision}\label{sec:stageC}

\paragraph{Aggregation.}
Let $\boldsymbol\mu\in\{0,1\}^{\nnz(\mat{H})}$ be a keep-mask over incidences and $\mat{H}^{\boldsymbol\mu}$ the corresponding masked incidence matrix, with hyperedge sizes $\mat{D}_e^{\boldsymbol\mu}=\diag((\mat{H}^{\boldsymbol\mu})^\top\one_N)$ and utility-weighted node degrees $\mat{D}_u^{\boldsymbol\mu}=\diag(\mat{H}^{\boldsymbol\mu}\vect{u})$. Node-to-hyperedge aggregation is the mean over retained members, and hyperedge-to-node aggregation is the utility-weighted mean over retained incident hyperedges:
\begin{equation}\label{eq:agg}
\mathcal{A}^{\boldsymbol\mu}(\mat{Y})=(\mat{D}_u^{\boldsymbol\mu})^{-1}\mat{H}^{\boldsymbol\mu}\diag(\vect{u})\,(\mat{D}_e^{\boldsymbol\mu})^{-1}(\mat{H}^{\boldsymbol\mu})^\top\mat{Y}.
\end{equation}
Nodes without retained incidences receive a zero aggregate. Only the non-zero values of the sparse operators depend on $\boldsymbol\mu$, so their index structure is built once.

\paragraph{Encoder.}
The encoder $f_{\boldsymbol\theta}$ receives the raw features with rows scaled to unit $\ell_2$ norm, $\hat{\mat{X}}$ (kept sparse when fewer than 25\% of its entries are non-zero), and the Stage-B summaries $\mat{X}_{\mathrm{side}}=[\mat{F}\,\|\,\mat{G}]$. $\Drop_p(\cdot)$ zeroes each entry independently with probability $p$ (without rescaling). With hidden width $h$, the first layer computes
\begin{align*}
\big[\mat{Y}^{(1)}_{\mathrm{s}}\,\|\,\mat{Y}^{(1)}_{\mathrm{n}}\big]&=\Drop_{p_f}(\hat{\mat{X}})\,\boldsymbol\Theta_1+\Drop_{p_f}(\mat{X}_{\mathrm{side}})\,\boldsymbol\Theta_{\mathrm{side}}+\one_N\vect{b}_1^\top\in\R^{N\times2h},\\
\mat{Q}&=\operatorname{GELU}\big(\operatorname{LN}\big(\mat{Y}^{(1)}_{\mathrm{s}}+\mathcal{A}^{\boldsymbol\mu}(\mat{Y}^{(1)}_{\mathrm{n}})\big)\big)\in\R^{N\times h},
\end{align*}
where the subscripts s and n denote the self and neighbour halves, LN is layer normalisation \citep{ba2016layernorm} and GELU the Gaussian error linear unit \citep{hendrycks2016gelu}. The second layer computes $[\mat{Y}^{(2)}_{\mathrm{s}}\,\|\,\mat{Y}^{(2)}_{\mathrm{n}}]=\mat{Q}\boldsymbol\Theta_2+\one_N\vect{b}_2^\top\in\R^{N\times 2d}$ and $\mat{Z}_0=\mat{Y}^{(2)}_{\mathrm{s}}+\mathcal{A}^{\boldsymbol\mu}(\mat{Y}^{(2)}_{\mathrm{n}})$. Applying the linear maps before aggregation is equivalent to aggregating first, because $\mathcal{A}^{\boldsymbol\mu}$ acts on rows and the weights on columns; for the bias terms this holds for every node with a retained incidence, since the rows of $\mathcal{A}^{\boldsymbol\mu}$ then sum to one. The output adds $K'$ parameter-free propagation steps,
\begin{equation}\label{eq:outprop}
\mat{Z}_k=\alpha'\mat{Z}_{k-1}+(1-\alpha')\mathcal{A}^{\boldsymbol\mu}(\mat{Z}_{k-1}),\quad k=1,\dots,K',\qquad
f_{\boldsymbol\theta}=\frac{1}{K'+1}\sum_{k=0}^{K'}\mat{Z}_k\in\R^{N\times d}.
\end{equation}
The parameters are $\boldsymbol\theta=\{\boldsymbol\Theta_1,\boldsymbol\Theta_{\mathrm{side}},\vect{b}_1,\boldsymbol\Theta_2,\vect{b}_2\}$ and those of LN.

\paragraph{Objective.}
In every epoch we draw two independent views, each with its own incidence keep-mask ($\mu_q\sim\mathrm{Bernoulli}(1-p_m)$ independently per incidence) and its own dropout masks, and obtain $\mat{Z}'$ and $\mat{Z}''$. Each is standardised column-wise, $\hat{\mat{Z}}=N^{-1/2}(\mat{Z}-\one_N\bar{\vect{z}}^\top)\diag(\vect{s})^{-1}$, with column means $\bar{\vect{z}}$ and column sample standard deviations $\vect{s}$ (plus $10^{-6}$), so that each column of $\hat{\mat{Z}}$ has squared norm $(N-1)/N$. The loss of CCA-SSG \citep{zhang2021ccassg} is
\begin{equation}\label{eq:loss}
\mathcal{L}(\boldsymbol\theta)=\big\|\hat{\mat{Z}}'-\hat{\mat{Z}}''\big\|_F^2+\lambda\Big(\big\|\hat{\mat{Z}}'^\top\hat{\mat{Z}}'-\mat{I}_d\big\|_F^2+\big\|\hat{\mat{Z}}''^\top\hat{\mat{Z}}''-\mat{I}_d\big\|_F^2\Big).
\end{equation}
The first term rewards per-dimension correlation between the two views and the second decorrelates the embedding dimensions, which prevents collapse without negative pairs, projection heads or target networks. Each evaluation of~\eqref{eq:loss} costs $O(Nd^2)$. We minimise $\mathcal{L}$ with full-batch Adam \citep{kingma2015adam} for a fixed number of epochs.

\paragraph{Output.}
After training, $f_{\boldsymbol\theta}$ is evaluated once on the uncorrupted hypergraph ($\boldsymbol\mu=\one$, no dropout); the result is centred column-wise and its rows are scaled to unit $\ell_2$ norm, giving $\mat{Z}$.

\paragraph{Relation to SE-HSSL.}
SE-HSSL also uses CCA-based objectives \citep{li2024sehssl}. HyperFuse differs in using a single node-level objective, no hyperedge- or membership-level terms, a fixed budget of 100 epochs, hyperedge-to-node aggregation weighted by Stage-B utilities, and precomputed structural and spectral inputs.

\begin{algorithm}[t]
\caption{HyperFuse}\label{alg:hyperfuse}
\begin{algorithmic}[1]
\Require incidence $\mat{H}$, features $\mat{X}$, width $d$, seed
\Ensure embeddings $\mat{Z}\in\R^{N\times d}$
\State remove hyperedges with $|e_j|<2$; set $d_s=d/4$, $d_g=d/4$, $d_f=d-d_s-d_g$
\Statex \textbf{Stage A}
\State $\mat{S}\gets\qf(\boldsymbol\Omega)$; choose $\eta$ (Proposition~\ref{prop:ascent}, Remark~\ref{rem:impl})
\Repeat\ $\mat{S}\gets\qf(\mat{S}+\eta\mat{B}\mat{S})$ \Comment{$\mat{B}\mat{S}$ via~\eqref{eq:matfree}}
\Until{relative change of $\tr(\mat{S}^\top\mat{B}\mat{S})<10^{-5}$ or 1{,}000 iterations}
\Statex \textbf{Stage B}
\State $\bar{\mat{X}}\gets$ row-normalised, centred $\mat{X}$; $\tilde{\mat{X}}\gets\bar{\mat{X}}-\beta\mat{S}\mat{S}^\top\bar{\mat{X}}$; $\mat{H}^{\mathrm{b}}\gets\MM_\rho(\mat{H})$
\State compute $\mat{Y}^{\mathrm{a}}_\ell,\mat{Y}^{\mathrm{b}}_\ell$, weights $\pi_{i\ell}$~\eqref{eq:scaleweights}; $\mat{F}\gets$ rank-$d_f$ SVD of $\mat{Y}^{\mathrm{cat}}$
\State $\mat{G}\gets$ rank-$d_g$ SVD of $\mat{R}$
\State compute targets $t_j$ and descriptors $\boldsymbol\psi_j$; fit ridge regression; $\vect{u}\gets$~\eqref{eq:utility}
\Statex \textbf{Stage C}
\For{epoch $=1,\dots,100$}
\State draw two views; $\mat{Z}',\mat{Z}''\gets f_{\boldsymbol\theta}$ on each view; update $\boldsymbol\theta$ with $\nabla\mathcal{L}$~\eqref{eq:loss}
\EndFor
\State $\mat{Z}\gets f_{\boldsymbol\theta}$ on the uncorrupted hypergraph; centre columns; normalise rows
\end{algorithmic}
\end{algorithm}

\subsection{Computational cost}\label{sec:complexity}

Let $n_H=\nnz(\mat{H})$ and $n_X=\nnz(\mat{X})$, and let $T_A\le1{,}000$ and $T_C=100$ be the numbers of Stage-A iterations and Stage-C epochs. By Proposition~\ref{prop:matfree}, one Stage-A iteration costs $O(n_Hd_s+Nd_s^2)$, the second term for the QR factorisation. Stage~B costs $O(Kn_HD+NDd_s)$ for propagation and residualisation, plus the truncated SVDs of the dense $N\times3D$ matrix $\mat{Y}^{\mathrm{cat}}$ and of $\mat{R}$, and $O(n_H(d_s+d_f+d_g))$ for the hyperedge descriptors; it stores dense $N\times D$ and $N\times3D$ arrays. One Stage-C epoch evaluates the encoder on two views at a cost of $O\big(n_Xh+N(d_f+d_g)h+n_H\big(h+(K'+1)d\big)+Nhd+Nd^2\big)$.
When $\min(N,M)\le3{,}000$, $\mat{G}$ is obtained by a dense SVD of $\mat{R}$, costing $O(NM\min(N,M))$ time and $O(NM)$ memory; otherwise LOBPCG is used.
For fixed widths, all stages are therefore linear in $n_H$ and $n_X$, apart from the dense feature operations of Stage~B, which are linear in $ND$, and the dense SVD of $\mat{R}$ used on small hypergraphs. By contrast, contrastive objectives that compare every node with all others, as in the official TriCL configurations without negative subsampling, require $O(N^2d)$ operations and $O(N^2)$ memory per epoch.

\section{Experimental setup}\label{sec:setup}

\subsection{Datasets}
We use nine public hypergraphs in the preprocessed form released with TriCL \citep{lee2023tricl} (Table~\ref{tab:datasets}): the co-citation hypergraphs Cora-CC, Citeseer and PubMed and the co-authorship hypergraphs Cora-CA and DBLP \citep{yadati2019hypergcn}; Zoo and 20News from the UCI repository \citep{dua2017uci}; and the 3D object datasets NTU2012 \citep{chen2003visual} and ModelNet40 \citep{wu2015shapenets}, with hypergraphs constructed from visual features as in \citet{feng2019hgnn}. In these versions, nodes that belong to no hyperedge are removed; Cora-CC, Citeseer, PubMed and Cora-CA therefore have fewer nodes than in their original releases (2{,}708, 3{,}312, 19{,}717 and 2{,}708 nodes). The datasets cover 101 to 41{,}302 nodes, hyperedges of up to 2{,}241 nodes, sparse binary, sparse real-valued and dense real-valued features, and 3 to 67 classes.

\begin{table}[t]
\caption{Dataset statistics. $n_A$ is the number of non-zeros that the Banerjee adjacency $\mat{A}$ would have if materialised. Features: B, binary or categorical; S, sparse real-valued; R, dense real-valued. Zoo contains one single-node hyperedge, which HyperFuse discards}\label{tab:datasets}
\begin{tabular}{@{}lrrrrrrrc@{}}
\toprule
Dataset & $N$ & $M$ & $\nnz(\mat{H})$ & $\max_j|e_j|$ & $n_A$ & $D$ & Classes & Feat. \\
\midrule
20News     & 16{,}242 & 100      & 65{,}451 & 2{,}241 & 53{,}268{,}400 & 100    & 4  & B \\
Citeseer   & 1{,}458  & 1{,}079  & 3{,}453  & 26      & 8{,}614        & 3{,}703 & 6  & B \\
Cora-CA    & 2{,}388  & 1{,}072  & 4{,}585  & 43      & 29{,}884       & 1{,}433 & 7  & B \\
Cora-CC    & 1{,}434  & 1{,}579  & 4{,}786  & 5       & 8{,}512        & 1{,}433 & 7  & B \\
DBLP       & 41{,}302 & 22{,}363 & 99{,}561 & 202     & 865{,}262      & 1{,}425 & 6  & B \\
ModelNet40 & 12{,}311 & 12{,}311 & 61{,}555 & 5       & 137{,}888      & 100    & 40 & R \\
NTU2012    & 2{,}012  & 2{,}012  & 10{,}060 & 5       & 20{,}026       & 100    & 67 & R \\
PubMed     & 3{,}840  & 7{,}963  & 34{,}629 & 171     & 247{,}654      & 500    & 3  & S \\
Zoo        & 101      & 43       & 1{,}717  & 93      & 10{,}100       & 16     & 7  & B \\
\botrule
\end{tabular}
\end{table}

\subsection{Baselines}
We compare with four self-supervised hypergraph embedding methods with public code: TriCL \citep{lee2023tricl}, SE-HSSL \citep{li2024sehssl}, VilLain \citep{lee2024villain} and HypeBoy \citep{kim2024hypeboy}. TriCL, SE-HSSL and HypeBoy use node features; VilLain uses only the hypergraph structure. Each method was run with the code and configurations released by its authors. For TriCL and SE-HSSL, this determines the number of epochs (TriCL 100--1{,}000, SE-HSSL 50--1{,}000) and the embedding width (TriCL 128--512, SE-HSSL 256--1{,}024). VilLain was run as in its official script: one model per number of virtual labels in $\{2,\dots,8\}$, each for up to 5{,}000 epochs with the original early stopping, followed by PCA to 128 dimensions. We kept the script's default of 100 propagation steps at inference on all datasets; the VilLain paper uses 10 on hypergraphs with fewer than 10{,}000 nodes. HypeBoy was run with its official code and configuration and outputs 128-dimensional embeddings. The only changes to the baseline code were initialising an early-stopping counter that the VilLain code referenced before assignment, and seeding every method with the benchmark seed. Table~\ref{tab:baselinecfg} in Appendix~\ref{app:config} lists the settings.

\subsection{HyperFuse configuration}
HyperFuse used a single default configuration on all datasets, with $d=128$ (Table~\ref{tab:hfconfig} in Appendix~\ref{app:config}); unlike TriCL and SE-HSSL, it has no per-dataset hyperparameters. The only data-dependent choice is the feature-dropout rate of Stage~C, which is set automatically from the feature type: $p_f=0.3$ for sparse or discrete features and $p_f=0$ for dense real-valued descriptors (NTU2012 and ModelNet40), for which entry masking removes information; the official TriCL configurations make a similar distinction for these two datasets.

\subsection{Evaluation protocol}
For every dataset, method and seed $\in\{0,1,2\}$, we generate one embedding and evaluate it as follows.

\paragraph{Classification.}
Following TriCL, we use random 10\%/10\%/80\% train/validation/test splits; the first three of TriCL's fixed splits are used for every method and seed, so that only the embedding changes between rows. Six downstream classifiers take the embedding as input features: logistic regression (LogReg) and a one-hidden-layer MLP with 256 units from scikit-learn \citep{pedregosa2011scikit}, trained on standardised embeddings of the training nodes, which ignore the hypergraph; and HGNN \citep{feng2019hgnn}, HNHN \citep{dong2020hnhn}, UniGCNII \citep{huang2021unignn} and AllSetTransformer \citep{chien2022allset} from the AllSet code base, which additionally use the hypergraph, trained for 500 epochs with the AllSet default hyperparameters and evaluated at the epoch of best validation accuracy. LogReg and the MLP are fitted on the training nodes only; the MLP uses scikit-learn's early stopping on 10\% of the training nodes, and neither uses the validation split. We report test accuracy and macro-F1 for each classifier, averaged over the three splits. As a reference, we also evaluate the raw node features in place of an embedding (Raw).

\paragraph{Clustering.}
We run $k$-means \citep{lloyd1982least} on the embedding with $k$ equal to the number of classes (10 initialisations, five $k$-means seeds) and report the mean adjusted Rand index \citep[ARI;][]{hubert1985comparing} and normalised mutual information \citep[NMI;][]{strehl2002cluster} with respect to the class labels.

\paragraph{Runtime.}
Embedding time is the wall-clock time of the embedding call, including all preprocessing performed inside each method and excluding dataset loading and downstream evaluation. All methods were run on the same workstation with an NVIDIA T1000 GPU with 8\,GB of memory. The baselines and Stage~C of HyperFuse ran on the GPU with PyTorch \citep{paszke2019pytorch}; Stages~A and B of HyperFuse ran on the CPU with NumPy and SciPy.

\paragraph{Aggregation and statistics.}
All scores are means over the three seeds; appendix tables also give standard deviations over seeds. SE-HSSL did not complete on DBLP within 24\,h (Sect.~\ref{sec:dblp}). We therefore report averages over the eight datasets on which all methods completed and report DBLP separately. For each classifier, and for NMI and ARI, we compare the five methods with the Friedman test on per-dataset ranks \citep{friedman1937use,demsar2006statistical}, and HyperFuse with each baseline with the two-sided exact Wilcoxon signed-rank test on the eight per-dataset scores \citep{wilcoxon1945individual}; we report uncorrected $p$-values. With eight datasets the smallest attainable $p$-value is $0.0078$, so these tests have limited power to detect small differences.

\section{Results}\label{sec:results}

\subsection{Embedding time}\label{sec:runtime}

Figure~\ref{fig:runtime} shows the embedding time of every method on every dataset, and the upper block of Table~\ref{tab:main} summarises it over the eight common datasets. HyperFuse was the fastest method on all nine datasets. It required between 0.7\,s (Zoo) and 17.7\,s (20News) on the eight common datasets, 8.7\,s on average and 1.2\,min in total. The baselines required 832.4\,min (TriCL), 335.6\,min (SE-HSSL), 225.7\,min (VilLain) and 157.9\,min (HypeBoy) in total. The geometric-mean speed-up of HyperFuse was 143$\times$ over TriCL, 96$\times$ over SE-HSSL, 179$\times$ over VilLain and 13$\times$ over HypeBoy, the fastest baseline. Relative to the fastest baseline on each dataset, the speed-up ranged from 1.7$\times$ on NTU2012, where HypeBoy needed 11.5\,s, to 114$\times$ on 20News, where the fastest baseline, VilLain, needed 34\,min.

The time of HyperFuse grew slowly with dataset size. PubMed (34{,}629 incidences) took 7.3\,s, compared with 9.1\,h for TriCL, 3.5\,h for SE-HSSL, 1.4\,h for VilLain and 61\,s for HypeBoy. On 20News, whose hyperedges contain up to 2{,}241 nodes, HyperFuse took 17.7\,s, compared with 2.3\,h for HypeBoy; the matrix-free operator of Proposition~\ref{prop:matfree} avoids the $53.3\times10^6$ non-zeros of the materialised adjacency. The embedding time of HyperFuse also varied little across seeds (Table~\ref{tab:app_time}).

\begin{figure}[t]
\centering
\includegraphics[width=\textwidth]{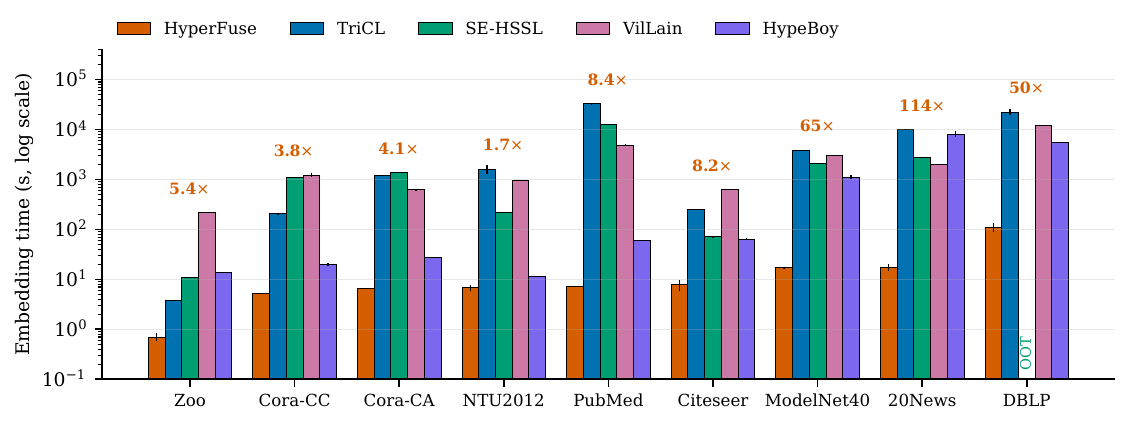}
\caption{Embedding time per dataset (mean over three seeds, error bars show one standard deviation; logarithmic scale). Datasets are ordered by the embedding time of HyperFuse. Numbers above the bars give the speed-up of HyperFuse over the fastest baseline on that dataset. OOT: SE-HSSL did not complete on DBLP within 24\,h}\label{fig:runtime}
\end{figure}

\begin{table}[t]
\caption{Summary over the eight datasets on which all methods completed (DBLP is reported in Table~\ref{tab:dblp}). Embedding time is given first; NMI and ARI are obtained with $k$-means, and accuracy and macro-F1 are reported per downstream classifier in Table~\ref{tab:clf}. The total time is summed over the eight datasets; speed-ups are those of HyperFuse over each baseline. Best in bold, second best underlined. Average ranks are computed per dataset (1 = best); the last column gives the $p$-value of the Friedman test}\label{tab:main}
\setlength{\tabcolsep}{3pt}\small
\begin{tabular}{@{}lcccccc@{}}
\toprule
 & HyperFuse & TriCL & SE-HSSL & VilLain & HypeBoy & \\
\midrule
\multicolumn{7}{@{}l}{\emph{Embedding time}} \\
Mean time (s) $\downarrow$ & \textbf{8.7} & 6242.8 & 2516.8 & 1693.0 & \underline{1184.4} & \\
Geo.-mean time (s) $\downarrow$ & \textbf{6.4} & 912.3 & 611.5 & 1145.5 & \underline{86.2} & \\
Total time (min) $\downarrow$ & \textbf{1.2} & 832.4 & 335.6 & 225.7 & \underline{157.9} & \\
Speed-up (geo.\ mean) & -- & 143$\times$ & 96$\times$ & 179$\times$ & 13$\times$ & \\
Speed-up (range) & -- & 5.4--4520$\times$ & 9.2--1720$\times$ & 82--675$\times$ & 1.7--461$\times$ & \\
\midrule
\multicolumn{7}{@{}l}{\emph{Clustering quality (mean over datasets)}} \\
NMI (\%) & \textbf{60.22} & \underline{60.03} & 59.96 & 48.38 & 52.82 & \\
ARI (\%) & \textbf{55.76} & 54.17 & \underline{54.58} & 39.21 & 39.65 & \\
\midrule
\multicolumn{6}{@{}l}{\emph{Average rank} $\downarrow$} & Friedman $p$ \\
NMI & 2.50 & \textbf{2.12} & 2.75 & 3.62 & 4.00 & 0.095 \\
ARI & 2.38 & \textbf{2.12} & 2.38 & 4.00 & 4.12 & 0.016 \\
\botrule
\end{tabular}
\end{table}

\subsection{Embedding quality}\label{sec:quality}

Table~\ref{tab:clf} reports accuracy and macro-F1 for each downstream classifier, averaged over the eight common datasets, with the raw node features as a reference. HyperFuse had the highest average accuracy with five of the six classifiers, all except the MLP, and the highest macro-F1 with the four hypergraph-aware classifiers; with LogReg and the MLP, TriCL had the highest macro-F1. With every classifier, HyperFuse, TriCL and SE-HSSL differed by at most 1.2 points in accuracy and 2.6 points in macro-F1, whereas HypeBoy was 2.1--4.1 points and VilLain, which does not use node features, 7.6--17.0 points less accurate than HyperFuse. HyperFuse also improved on the raw features with every classifier, most with LogReg and the MLP, which see the hypergraph only through the embedding. Appendix~\ref{app:ablation} reports an ablation of the encoder's side input $[\mat{F}\mid\mat{G}]$, evaluated with logistic regression on the embeddings, and Appendix~\ref{app:featdeg} compares HyperFuse with its encoder alone when the node features are degraded.

\begin{table}[t]
\caption{Accuracy and macro-F1 (\%) per downstream classifier, averaged over the eight common datasets and three seeds. Raw: the node features used in place of an embedding (reference; for the last four classifiers this is their standard semi-supervised use). LogReg and MLP ignore the hypergraph; the other four classifiers also use it. Best in bold, second best underlined, among the five embedding methods. HF: HyperFuse, TC: TriCL, SH: SE-HSSL, VL: VilLain, HB: HypeBoy}\label{tab:clf}
\setlength{\tabcolsep}{2.6pt}\footnotesize
\begin{tabular}{@{}lcccccccccccc@{}}
\toprule
 & \multicolumn{6}{c}{Accuracy} & \multicolumn{6}{c}{Macro-F1} \\
\cmidrule(lr){2-7}\cmidrule(l){8-13}
Classifier & Raw & HF & TC & SH & VL & HB & Raw & HF & TC & SH & VL & HB \\
\midrule
LogReg & 70.63 & \textbf{80.37} & \underline{80.08} & 79.18 & 64.69 & 77.88 & 63.29 & \underline{73.19} & \textbf{74.08} & 73.03 & 58.59 & 71.41 \\
MLP & 64.46 & 78.61 & \textbf{79.49} & \underline{79.22} & 61.57 & 74.54 & 57.40 & 70.47 & \textbf{73.06} & \underline{72.52} & 55.47 & 66.39 \\
HGNN & 77.26 & \textbf{79.01} & \underline{78.49} & 78.21 & 68.22 & 75.32 & 69.31 & \textbf{70.90} & \underline{70.68} & 70.22 & 59.08 & 65.71 \\
HNHN & 73.45 & \textbf{78.85} & \underline{78.59} & 78.31 & 70.97 & 75.91 & 65.70 & \textbf{71.12} & \underline{70.61} & 70.60 & 61.99 & 67.92 \\
UniGCNII & 78.98 & \textbf{79.84} & \underline{79.08} & 78.73 & 71.53 & 76.19 & 71.93 & \textbf{72.26} & \underline{71.42} & 71.34 & 62.38 & 68.28 \\
AllSetTr. & 76.73 & \textbf{80.00} & \underline{79.77} & 79.33 & 72.40 & 77.89 & 69.10 & \textbf{72.70} & \underline{72.54} & 72.23 & 64.22 & 70.06 \\
\botrule
\end{tabular}
\end{table}

For clustering (Table~\ref{tab:main}), HyperFuse attained an NMI of 60.22\% and an ARI of 55.76\%, compared with 60.03\% and 54.17\% for TriCL, 59.96\% and 54.58\% for SE-HSSL, 48.38\% and 39.21\% for VilLain, and 52.82\% and 39.65\% for HypeBoy. No method dominated across datasets (Tables~\ref{tab:app_nmi} and~\ref{tab:app_ari}): HyperFuse had the highest NMI and ARI on 20News, Cora-CC and NTU2012, and TriCL on most of the other datasets. HyperFuse was clearly behind on some clustering tasks: its ARI on ModelNet40 was 87.10\%, against 94.33\% for TriCL and 94.52\% for SE-HSSL, and its NMI on Zoo was 86.18\%, against 91.56\% for TriCL. Conversely, its ARI on 20News was 47.42\%, against 20.43\% for TriCL and 39.98\% for SE-HSSL. VilLain, which does not use node features, had the best NMI and ARI on PubMed but was otherwise less accurate than the feature-based methods. Figure~\ref{fig:qt} places these results against the geometric-mean embedding time.

\begin{figure}[t]
\centering
\includegraphics[width=\textwidth]{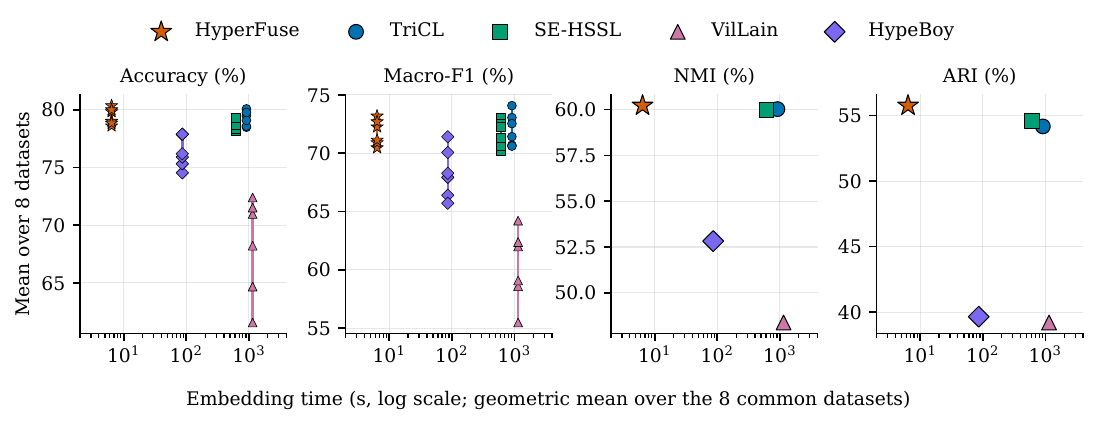}
\caption{Embedding quality against geometric-mean embedding time over the eight common datasets. For accuracy and macro-F1, each marker is one downstream classifier (score averaged over the datasets), and the vertical line spans the six classifiers. HyperFuse attains quality comparable to TriCL and SE-HSSL at about two orders of magnitude lower embedding time, and higher quality than HypeBoy, the fastest baseline, at about one order of magnitude lower embedding time}\label{fig:qt}
\end{figure}

\paragraph{Statistical comparison.}
With every classifier, the Friedman test indicates differences among the five methods in accuracy and macro-F1 (all $p\le0.006$; Table~\ref{tab:ranks}); for clustering, it indicates differences in ARI ($p=0.016$) but not in NMI ($p=0.095$). These differences are due to VilLain and HypeBoy. HyperFuse was more accurate than VilLain with every classifier (better on at least seven of the eight datasets, $p\le0.016$) and more accurate than HypeBoy with every classifier (better on at least six of the eight datasets, $p\le0.039$; for macro-F1, $p\le0.016$ with the four hypergraph-aware classifiers and $p=0.15$ with LogReg and the MLP). Of its 24 comparisons with TriCL and SE-HSSL (six classifiers, two scores), only one had $p<0.05$ (accuracy with UniGCNII against SE-HSSL, $p=0.039$, in favour of HyperFuse). For clustering, HyperFuse had higher NMI and ARI than HypeBoy on seven of the eight datasets ($p=0.023$ for both), whereas none of its clustering comparisons with TriCL, SE-HSSL and VilLain was significant at the 0.05 level (all $p\ge0.055$). We conclude that, on these benchmarks, HyperFuse matches the classification and clustering quality of TriCL and SE-HSSL and exceeds that of HypeBoy and VilLain; the data do not support a claim that it is more accurate than TriCL or SE-HSSL.

\begin{table}[t]
\caption{Average ranks (1 = best) of the five methods over the eight common datasets for each downstream classifier, with the $p$-value of the Friedman test. Best (lowest) rank in bold. HF: HyperFuse, TC: TriCL, SH: SE-HSSL, VL: VilLain, HB: HypeBoy}\label{tab:ranks}
\setlength{\tabcolsep}{2.8pt}\footnotesize
\begin{tabular}{@{}lcccccccccccc@{}}
\toprule
 & \multicolumn{6}{c}{Accuracy} & \multicolumn{6}{c}{Macro-F1} \\
\cmidrule(lr){2-7}\cmidrule(l){8-13}
Classifier & HF & TC & SH & VL & HB & $p$ & HF & TC & SH & VL & HB & $p$ \\
\midrule
LogReg & \textbf{1.75} & 2.00 & 3.12 & 4.75 & 3.38 & $<$0.001 & \textbf{2.00} & \textbf{2.00} & 3.00 & 4.75 & 3.25 & 0.003 \\
MLP & 2.50 & \textbf{1.75} & 2.25 & 4.88 & 3.62 & $<$0.001 & 2.88 & \textbf{1.75} & 2.00 & 4.88 & 3.50 & $<$0.001 \\
HGNN & 2.25 & \textbf{1.62} & 2.12 & 4.75 & 4.25 & $<$0.001 & 2.50 & \textbf{1.50} & 2.12 & 4.75 & 4.12 & $<$0.001 \\
HNHN & 1.88 & \textbf{1.75} & 2.62 & 4.62 & 4.12 & $<$0.001 & 2.25 & \textbf{1.88} & 2.25 & 4.62 & 4.00 & $<$0.001 \\
UniGCNII & \textbf{1.50} & 1.88 & 2.62 & 4.75 & 4.25 & $<$0.001 & \textbf{1.62} & 2.00 & 2.38 & 4.88 & 4.12 & $<$0.001 \\
AllSetTr. & 2.00 & \textbf{1.88} & 3.12 & 4.38 & 3.62 & 0.006 & 2.25 & \textbf{1.75} & 2.50 & 4.62 & 3.88 & $<$0.001 \\
\botrule
\end{tabular}
\end{table}
\subsection{Results on DBLP}\label{sec:dblp}

Table~\ref{tab:dblp} reports DBLP separately because SE-HSSL did not complete within 24\,h on the NVIDIA T1000 (8\,GB) GPU in any of the three seeds (OOT). HyperFuse generated the DBLP embedding in 110.4\,s, compared with 6.2\,h for TriCL (202$\times$), 3.4\,h for VilLain (110$\times$) and 1.5\,h for HypeBoy (50$\times$). With every downstream classifier, its accuracy was within 3.7 points of the most accurate method (0.6--3.7 points below TriCL and 0.3--2.0 points below HypeBoy), and it was 8.6--14.3 points more accurate than VilLain.

\begin{table}[t]
\caption{Results on DBLP (mean $\pm$ standard deviation over three seeds). SE-HSSL did not complete within 24\,h on the NVIDIA T1000 (8\,GB) GPU in any seed. Best in bold, second best underlined}\label{tab:dblp}
\setlength{\tabcolsep}{3pt}\footnotesize
\begin{tabular}{@{}llcccc@{}}
\toprule
 & & HyperFuse & TriCL & VilLain & HypeBoy \\
\midrule
Time (s) $\downarrow$ &  & \textbf{110.4{\scriptsize$\pm$20.9}} & 22322.3{\scriptsize$\pm$3053.9} & 12170.2{\scriptsize$\pm$193.0} & \underline{5572.0{\scriptsize$\pm$99.7}} \\
\midrule
\multirow{6}{*}{Accuracy (\%)} & LogReg & 89.39{\scriptsize$\pm$0.06} & \underline{89.99{\scriptsize$\pm$0.12}} & 75.08{\scriptsize$\pm$0.16} & \textbf{90.27{\scriptsize$\pm$0.02}} \\
 & MLP & 88.88{\scriptsize$\pm$0.15} & \textbf{91.10{\scriptsize$\pm$0.06}} & 78.23{\scriptsize$\pm$0.65} & \underline{89.57{\scriptsize$\pm$0.20}} \\
 & HGNN & 86.80{\scriptsize$\pm$0.16} & \textbf{90.50{\scriptsize$\pm$0.02}} & 77.27{\scriptsize$\pm$0.12} & \underline{88.78{\scriptsize$\pm$0.01}} \\
 & HNHN & 89.23{\scriptsize$\pm$0.18} & \textbf{91.22{\scriptsize$\pm$0.05}} & 79.06{\scriptsize$\pm$0.16} & \underline{89.61{\scriptsize$\pm$0.09}} \\
 & UniGCNII & 88.26{\scriptsize$\pm$0.09} & \textbf{91.02{\scriptsize$\pm$0.07}} & 79.63{\scriptsize$\pm$0.08} & \underline{89.34{\scriptsize$\pm$0.04}} \\
 & AllSetTr. & 89.77{\scriptsize$\pm$0.07} & \textbf{91.31{\scriptsize$\pm$0.05}} & 79.49{\scriptsize$\pm$0.15} & \underline{90.02{\scriptsize$\pm$0.07}} \\
\midrule
\multirow{6}{*}{Macro-F1 (\%)} & LogReg & 88.86{\scriptsize$\pm$0.11} & \underline{89.53{\scriptsize$\pm$0.13}} & 72.93{\scriptsize$\pm$0.20} & \textbf{89.89{\scriptsize$\pm$0.01}} \\
 & MLP & 88.25{\scriptsize$\pm$0.21} & \textbf{90.74{\scriptsize$\pm$0.04}} & 76.62{\scriptsize$\pm$0.78} & \underline{89.15{\scriptsize$\pm$0.25}} \\
 & HGNN & 85.93{\scriptsize$\pm$0.14} & \textbf{90.18{\scriptsize$\pm$0.02}} & 75.45{\scriptsize$\pm$0.17} & \underline{88.36{\scriptsize$\pm$0.03}} \\
 & HNHN & 88.60{\scriptsize$\pm$0.21} & \textbf{90.88{\scriptsize$\pm$0.05}} & 77.60{\scriptsize$\pm$0.12} & \underline{89.22{\scriptsize$\pm$0.11}} \\
 & UniGCNII & 87.57{\scriptsize$\pm$0.09} & \textbf{90.72{\scriptsize$\pm$0.09}} & 78.20{\scriptsize$\pm$0.14} & \underline{88.93{\scriptsize$\pm$0.03}} \\
 & AllSetTr. & 89.27{\scriptsize$\pm$0.08} & \textbf{90.98{\scriptsize$\pm$0.05}} & 78.04{\scriptsize$\pm$0.10} & \underline{89.66{\scriptsize$\pm$0.06}} \\
\midrule
NMI (\%) &  & 55.21{\scriptsize$\pm$0.89} & \textbf{62.33{\scriptsize$\pm$1.45}} & 37.99{\scriptsize$\pm$0.50} & \underline{59.42{\scriptsize$\pm$1.33}} \\
ARI (\%) &  & 42.52{\scriptsize$\pm$1.52} & \underline{48.88{\scriptsize$\pm$2.33}} & 30.14{\scriptsize$\pm$0.70} & \textbf{49.34{\scriptsize$\pm$1.23}} \\
\botrule
\end{tabular}
\end{table}
\section{Discussion and limitations}\label{sec:discussion}

The main finding is that, on the eight benchmark hypergraphs on which all methods completed, the embeddings of HyperFuse supported downstream classification and clustering about as well as those of TriCL and SE-HSSL, while being generated two orders of magnitude faster, and better than those of HypeBoy, the fastest baseline, while being generated 13$\times$ faster on geometric mean. Two properties account for the speed. First, Stages~A and B use only linear algebra whose cost scales with the number of incidences rather than with the number of node pairs, and they never materialise the hypergraph adjacency. Second, Stage~C trains a small encoder for 100 epochs with an objective whose cost is $O(Nd^2)$ rather than $O(N^2d)$. A further practical property is that HyperFuse used one configuration for all datasets without per-dataset tuning, whereas TriCL and SE-HSSL used per-dataset configurations provided by their authors. Stages~A and~B matter most when node features are weak: with 5\% of the feature columns, HyperFuse was on average 5.2 points more accurate than its encoder alone with logistic regression, with the gains concentrated on five of the eight datasets (Appendix~\ref{app:featdeg}).

The time saving matters most when many embeddings are needed. As an illustration only, with the measured PubMed times and ignoring parallelism, embedding a collection of 100 hypergraphs of this size would take about 12\,min with HyperFuse, compared with about 38 days with TriCL, 14 days with SE-HSSL, 6 days with VilLain and 1.7\,h with HypeBoy. This arithmetic assumes that all members of the collection resemble PubMed and that times add up; it is not a measurement.

\paragraph{Towards evolving hypergraphs.} Evolving hypergraphs motivated this work, and although we have not evaluated HyperFuse on them, its design lends itself to that setting. At a cost of seconds per embedding, re-embedding each snapshot from scratch is already practical for hypergraphs of the sizes studied here. Beyond that, Stage~A is a subspace iteration that can be initialised with the previous snapshot's $\mat{S}$ instead of a random basis; the Stage-C encoder is a parametric message-passing network whose weights can be carried over and fine-tuned on the next snapshot; and the Stage-C operators are rebuilt from the incidence structure alone. Two parts would need adaptation: the spectral summaries $\mat{F}$ and $\mat{G}$, whose bases can change between snapshots, and the rotational invariance of the embeddings, which requires aligning successive snapshots.

The study has the following limitations:

\emph{Scope of the evaluation.} We did not evaluate HyperFuse on collections of patient-specific networks or on evolving hypergraphs, which motivated this work; our evidence consists of per-hypergraph embedding times and downstream quality on standard benchmarks. HyperFuse has no incremental update: a new snapshot is embedded from scratch. Embeddings of different hypergraphs, or of different snapshots of one hypergraph, are computed independently and are defined only up to the invariances of the objective, so comparisons across graphs require an alignment step. Apart from the artificial feature degradation of Appendix~\ref{app:featdeg}, all results are on hypergraphs with informative node features; on the citation hypergraphs, the structural stages did not help when the features were weak. We evaluated node classification and clustering only, not hyperedge prediction or inductive settings.

\emph{Comparability of runtimes.} Runtimes were measured on one workstation with an 8\,GB consumer GPU, and the baselines ran with their official epoch budgets and embedding widths, which for TriCL and SE-HSSL differ from those of HyperFuse (128 dimensions); a baseline trained for fewer epochs would be faster, possibly at lower quality. Stages~A and B of HyperFuse ran on the CPU while all other training ran on the GPU. Absolute times will differ on other hardware, and the time-out of SE-HSSL on DBLP applies to the hardware used here.

\emph{Memory.} Stage~B stores dense $N\times D$ and $N\times3D$ feature arrays, so its memory grows with the product of the number of nodes and the feature dimension. This did not limit the present benchmarks but would matter for very high-dimensional features.

\emph{Statistical power.} Eight datasets, three seeds and three of TriCL's 20 splits provide limited power to detect differences of the size observed between HyperFuse, TriCL and SE-HSSL; our conclusion is that no difference was detected, not that the methods are equivalent.

\section{Conclusion}\label{sec:conclusion}

HyperFuse generates self-supervised node embeddings for attributed hypergraphs by combining matrix-free modularity maximisation on Banerjee's hypergraph adjacency, corruption-consistent feature summaries with bounded hyperedge utility weights, and a short correlation-based training stage. On eight benchmark hypergraphs, it generated embeddings 13--179$\times$ faster (geometric mean) than four self-supervised baselines, with classification and clustering quality not significantly different from those of the two strongest and higher than that of the fastest baseline, HypeBoy. Because a new embedding costs seconds rather than hours, HyperFuse is suited to analyses that require many embeddings, such as collections of patient-specific networks, and it is a natural starting point for embedding evolving hypergraphs. Evaluating it in those settings is the next step.

\backmatter


\section*{Declarations}

\begin{itemize}
\item \textbf{Funding:} The work of the author Anirban Banerjee
is supported by the Department of Biotechnology, India (No. BT/PR40182/BTIS/137/88/2025).
\item \textbf{Competing interests:} The authors declare no competing interests.
\item \textbf{Ethics approval and consent to participate:} Not applicable.
\item \textbf{Consent for publication:} Not applicable.
\item \textbf{Data availability:} All datasets are publicly available in the preprocessed form released with TriCL \citep{lee2023tricl}.
\item \textbf{Code availability:} The code of HyperFuse and of the benchmark is available at \url{https://github.com/ideas-are-bulletproof/HyperFuse_benchmark}.
\end{itemize}

\bibliography{hyperfuse_refs}

\begin{appendices}

\section{Configurations}\label{app:config}

\begin{table}[htbp!]
\caption{Default configuration of HyperFuse, used unchanged on all datasets ($d=128$)}\label{tab:hfconfig}
\begin{tabular}{@{}lp{5.2cm}p{5.6cm}@{}}
\toprule
Stage & Setting & Value \\
\midrule
A & structural width $d_s$ & $d/4=32$ \\
  & step-size factor $c$; maximum iterations & $0.9$; $1{,}000$ \\
  & stopping rule & rel.\ change $<10^{-5}$, checked every 5 iterations, patience 5, min.\ 20 \\
\midrule
B & feature and incidence widths $d_f$, $d_g$ & $64$, $32$ \\
  & residualisation $\beta$; masking rate $\rho$ & $0.5$; $0.10$ \\
  & laziness $\alpha$; hops $K$ & $0.5$; $3$ \\
  & scale-weight temperature $\tau$; floor $\omega$ & $0.25$; $0.05$ \\
  & ridge: steps, step size, penalty & $160$, $0.03$, $0.05$ \\
  & utility temperature; bounds $[u_{\min},u_{\max}]$ & $0.8$; $[0.65,1.35]$ \\
\midrule
C & hidden width $h$; output width $d$ & $512$; $128$ \\
  & epochs; optimiser; learning rate; weight decay & $100$; Adam; $10^{-3}$; $0$ \\
  & decorrelation weight $\lambda$ & $10^{-3}$ \\
  & feature dropout $p_f$ & $0.3$ (sparse/discrete), $0$ (dense real-valued) \\
  & membership dropout $p_m$ & $0.3$ \\
  & output propagation: steps $K'$, laziness $\alpha'$ & $2$, $0.5$ \\
\botrule
\end{tabular}
\end{table}

\begin{table}[htbp!]
\caption{Baseline settings taken from the official configurations: number of training epochs and embedding width per dataset. VilLain trains seven models (2--8 virtual labels), each for up to 5{,}000 epochs with early stopping; VilLain and HypeBoy output 128 dimensions on all datasets; HypeBoy uses its repository's default configuration on all datasets}\label{tab:baselinecfg}
\begin{tabular}{@{}lcccc@{}}
\toprule
 & \multicolumn{2}{c}{TriCL} & \multicolumn{2}{c}{SE-HSSL} \\
\cmidrule(lr){2-3}\cmidrule(l){4-5}
Dataset & Epochs & Width & Epochs & Width \\
\midrule
20News     & 500   & 256 & 50    & 512 \\
Citeseer   & 500   & 512 & 265   & 256 \\
Cora-CA    & 800   & 512 & 160   & 512 \\
Cora-CC    & 300   & 512 & 305   & 512 \\
DBLP       & 500   & 256 & 60    & 512 \\
ModelNet40 & 200   & 256 & 50    & 512 \\
NTU2012    & 200   & 512 & 60    & 512 \\
PubMed     & 1{,}000 & 512 & 1{,}000 & 512 \\
Zoo        & 100   & 128 & 50    & 1{,}024 \\
\botrule
\end{tabular}
\end{table}

\section{Dataset-wise results}\label{app:results}
\setcounter{table}{0}\setcounter{figure}{0}

All values are means $\pm$ standard deviations over three seeds. OOT: did not complete within 24\,h. Best in bold, second best underlined. Accuracy and macro-F1 for each downstream classifier are given in Appendix~\ref{app:clf}.

\begin{table}[htbp!]
\caption{Embedding time (s) per dataset, and speed-up of HyperFuse over each baseline}\label{tab:app_time}
\setlength{\tabcolsep}{3pt}\footnotesize
\begin{tabular}{@{}lccccc@{}}
\toprule
Dataset & HyperFuse & TriCL & SE-HSSL & VilLain & HypeBoy \\
\midrule
20News & \textbf{17.7{\scriptsize$\pm$2.9}} & 10010.4{\scriptsize$\pm$14.4} & 2751.3{\scriptsize$\pm$0.3} & \underline{2021.5{\scriptsize$\pm$4.1}} & 8157.0{\scriptsize$\pm$1105.0} \\
Citeseer & \textbf{7.8{\scriptsize$\pm$1.8}} & 253.3{\scriptsize$\pm$0.1} & 71.8{\scriptsize$\pm$0.2} & 634.8{\scriptsize$\pm$21.5} & \underline{64.1{\scriptsize$\pm$2.7}} \\
Cora-CA & \textbf{6.6{\scriptsize$\pm$0.2}} & 1182.6{\scriptsize$\pm$30.6} & 1380.0{\scriptsize$\pm$1.2} & 624.5{\scriptsize$\pm$2.1} & \underline{27.2{\scriptsize$\pm$0.4}} \\
Cora-CC & \textbf{5.2{\scriptsize$\pm$0.1}} & 205.3{\scriptsize$\pm$2.1} & 1114.5{\scriptsize$\pm$0.8} & 1228.9{\scriptsize$\pm$89.9} & \underline{19.7{\scriptsize$\pm$1.4}} \\
DBLP & \textbf{110.4{\scriptsize$\pm$20.9}} & 22322.3{\scriptsize$\pm$3053.9} & OOT & 12170.2{\scriptsize$\pm$193.0} & \underline{5572.0{\scriptsize$\pm$99.7}} \\
ModelNet40 & \textbf{17.2{\scriptsize$\pm$0.7}} & 3872.1{\scriptsize$\pm$102.1} & 2107.1{\scriptsize$\pm$20.9} & 2975.7{\scriptsize$\pm$62.7} & \underline{1120.7{\scriptsize$\pm$111.6}} \\
NTU2012 & \textbf{6.8{\scriptsize$\pm$1.0}} & 1624.6{\scriptsize$\pm$361.4} & 222.4{\scriptsize$\pm$1.8} & 941.3{\scriptsize$\pm$24.3} & \underline{11.5{\scriptsize$\pm$0.1}} \\
PubMed & \textbf{7.3{\scriptsize$\pm$0.1}} & 32790.7{\scriptsize$\pm$394.5} & 12476.2{\scriptsize$\pm$0.7} & 4897.6{\scriptsize$\pm$118.7} & \underline{61.3{\scriptsize$\pm$0.8}} \\
Zoo & \textbf{0.7{\scriptsize$\pm$0.1}} & \underline{3.8{\scriptsize$\pm$0.1}} & 11.0{\scriptsize$\pm$0.1} & 219.4{\scriptsize$\pm$0.7} & 13.7{\scriptsize$\pm$0.0} \\
\midrule
 & \multicolumn{5}{c}{Speed-up of HyperFuse} \\
\cmidrule(l){2-6}
Dataset & -- & TriCL & SE-HSSL & VilLain & HypeBoy \\
\midrule
20News & -- & 566$\times$ & 155$\times$ & 114$\times$ & 461$\times$ \\
Citeseer & -- & 33$\times$ & 9.2$\times$ & 82$\times$ & 8.2$\times$ \\
Cora-CA & -- & 180$\times$ & 210$\times$ & 95$\times$ & 4.1$\times$ \\
Cora-CC & -- & 39$\times$ & 214$\times$ & 236$\times$ & 3.8$\times$ \\
DBLP & -- & 202$\times$ & OOT & 110$\times$ & 50$\times$ \\
ModelNet40 & -- & 225$\times$ & 123$\times$ & 173$\times$ & 65$\times$ \\
NTU2012 & -- & 238$\times$ & 33$\times$ & 138$\times$ & 1.7$\times$ \\
PubMed & -- & 4520$\times$ & 1720$\times$ & 675$\times$ & 8.4$\times$ \\
Zoo & -- & 5.4$\times$ & 16$\times$ & 314$\times$ & 20$\times$ \\
\botrule
\end{tabular}
\end{table}

\begin{table}[htbp!]
\caption{NMI (\%) per dataset ($k$-means)}\label{tab:app_nmi}
\setlength{\tabcolsep}{3pt}\footnotesize
\begin{tabular}{@{}lccccc@{}}
\toprule
Dataset & HyperFuse & TriCL & SE-HSSL & VilLain & HypeBoy \\
\midrule
20News & \textbf{43.59{\scriptsize$\pm$0.70}} & 31.36{\scriptsize$\pm$3.68} & 38.17{\scriptsize$\pm$0.22} & \underline{38.91{\scriptsize$\pm$0.12}} & 31.61{\scriptsize$\pm$1.48} \\
Citeseer & 42.19{\scriptsize$\pm$1.16} & \textbf{44.12{\scriptsize$\pm$0.23}} & \underline{42.80{\scriptsize$\pm$2.27}} & 24.93{\scriptsize$\pm$0.31} & 39.01{\scriptsize$\pm$0.51} \\
Cora-CA & \underline{49.04{\scriptsize$\pm$3.02}} & \textbf{50.03{\scriptsize$\pm$1.28}} & 46.53{\scriptsize$\pm$0.67} & 9.43{\scriptsize$\pm$0.87} & 44.76{\scriptsize$\pm$0.80} \\
Cora-CC & \textbf{55.68{\scriptsize$\pm$0.66}} & \underline{55.20{\scriptsize$\pm$0.87}} & 54.35{\scriptsize$\pm$1.66} & 36.56{\scriptsize$\pm$1.55} & 42.68{\scriptsize$\pm$0.47} \\
DBLP & 55.21{\scriptsize$\pm$0.89} & \textbf{62.33{\scriptsize$\pm$1.45}} & OOT & 37.99{\scriptsize$\pm$0.50} & \underline{59.42{\scriptsize$\pm$1.33}} \\
ModelNet40 & 93.75{\scriptsize$\pm$0.27} & \textbf{95.54{\scriptsize$\pm$0.15}} & \underline{95.51{\scriptsize$\pm$0.27}} & 75.23{\scriptsize$\pm$0.28} & 79.32{\scriptsize$\pm$0.17} \\
NTU2012 & \textbf{83.69{\scriptsize$\pm$0.09}} & \underline{83.09{\scriptsize$\pm$0.11}} & 82.80{\scriptsize$\pm$0.08} & 78.96{\scriptsize$\pm$0.18} & 72.51{\scriptsize$\pm$0.21} \\
PubMed & 27.65{\scriptsize$\pm$2.28} & 29.33{\scriptsize$\pm$0.46} & 29.67{\scriptsize$\pm$3.24} & \textbf{31.87{\scriptsize$\pm$0.63}} & \underline{31.05{\scriptsize$\pm$0.64}} \\
Zoo & 86.18{\scriptsize$\pm$1.21} & \textbf{91.56{\scriptsize$\pm$1.07}} & 89.88{\scriptsize$\pm$0.76} & \underline{91.14{\scriptsize$\pm$0.03}} & 81.64{\scriptsize$\pm$0.71} \\
\botrule
\end{tabular}
\end{table}

\begin{table}[htbp!]
\caption{ARI (\%) per dataset ($k$-means)}\label{tab:app_ari}
\setlength{\tabcolsep}{3pt}\footnotesize
\begin{tabular}{@{}lccccc@{}}
\toprule
Dataset & HyperFuse & TriCL & SE-HSSL & VilLain & HypeBoy \\
\midrule
20News & \textbf{47.42{\scriptsize$\pm$1.45}} & 20.43{\scriptsize$\pm$5.46} & \underline{39.98{\scriptsize$\pm$0.52}} & 38.86{\scriptsize$\pm$0.12} & 15.98{\scriptsize$\pm$2.18} \\
Citeseer & 41.66{\scriptsize$\pm$0.64} & \textbf{45.02{\scriptsize$\pm$0.31}} & \underline{42.02{\scriptsize$\pm$1.74}} & 17.78{\scriptsize$\pm$0.70} & 38.63{\scriptsize$\pm$1.51} \\
Cora-CA & \underline{44.30{\scriptsize$\pm$5.56}} & \textbf{45.96{\scriptsize$\pm$3.13}} & 38.52{\scriptsize$\pm$3.29} & 2.29{\scriptsize$\pm$0.76} & 37.15{\scriptsize$\pm$2.63} \\
Cora-CC & \textbf{50.05{\scriptsize$\pm$0.67}} & \underline{49.04{\scriptsize$\pm$0.12}} & 45.55{\scriptsize$\pm$3.60} & 23.99{\scriptsize$\pm$2.26} & 27.75{\scriptsize$\pm$1.86} \\
DBLP & 42.52{\scriptsize$\pm$1.52} & \underline{48.88{\scriptsize$\pm$2.33}} & OOT & 30.14{\scriptsize$\pm$0.70} & \textbf{49.34{\scriptsize$\pm$1.23}} \\
ModelNet40 & 87.10{\scriptsize$\pm$0.78} & \underline{94.33{\scriptsize$\pm$0.86}} & \textbf{94.52{\scriptsize$\pm$0.77}} & 55.61{\scriptsize$\pm$0.65} & 56.59{\scriptsize$\pm$0.79} \\
NTU2012 & \textbf{68.53{\scriptsize$\pm$0.99}} & \underline{67.06{\scriptsize$\pm$0.48}} & 65.92{\scriptsize$\pm$0.59} & 55.23{\scriptsize$\pm$0.76} & 38.74{\scriptsize$\pm$0.63} \\
PubMed & 22.67{\scriptsize$\pm$0.48} & 24.22{\scriptsize$\pm$0.41} & 25.27{\scriptsize$\pm$4.70} & \textbf{37.15{\scriptsize$\pm$0.56}} & \underline{28.43{\scriptsize$\pm$1.28}} \\
Zoo & 84.33{\scriptsize$\pm$0.69} & \textbf{87.31{\scriptsize$\pm$0.58}} & \underline{84.85{\scriptsize$\pm$3.09}} & 82.76{\scriptsize$\pm$0.53} & 73.92{\scriptsize$\pm$4.37} \\
\botrule
\end{tabular}
\end{table}

\begin{figure}[htbp!]
\centering
\includegraphics[width=\textwidth]{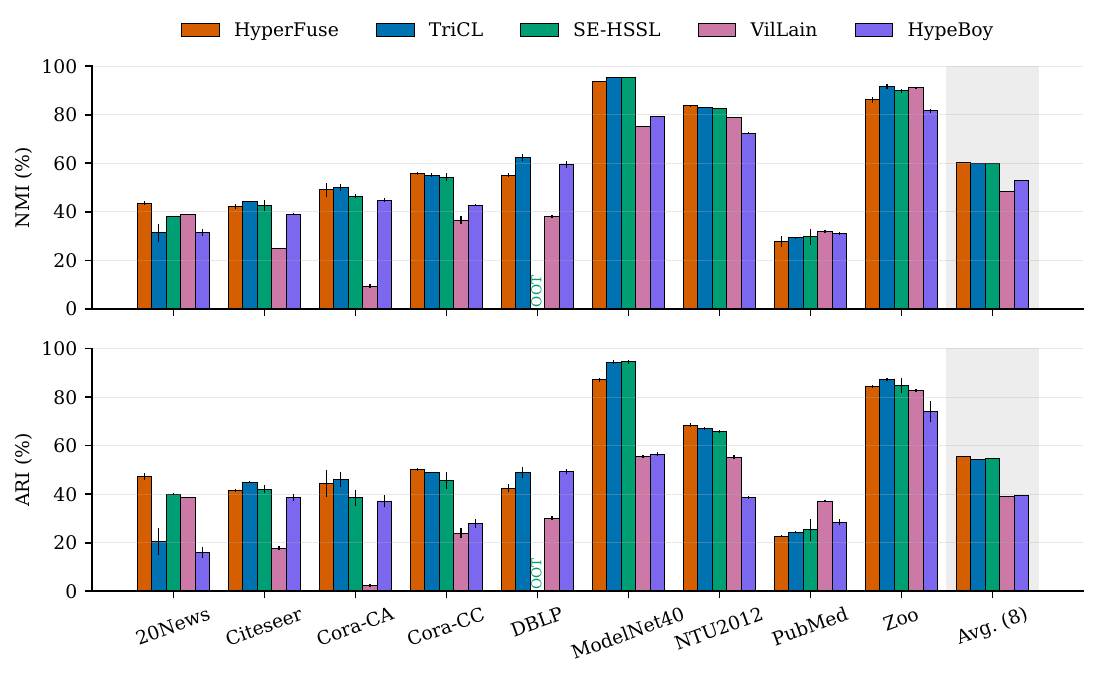}
\caption{Per-dataset NMI and ARI ($k$-means; mean over three seeds, error bars show one standard deviation). Avg.~(8): mean over the eight datasets other than DBLP}\label{fig:perdataset}
\end{figure}

\clearpage
\section{Results per downstream classifier}\label{app:clf}
\setcounter{table}{0}\setcounter{figure}{0}

Mean $\pm$ standard deviation over three seeds of the test score averaged over three splits. AllSetTr.: AllSetTransformer. Raw: the raw node features used in place of an embedding, evaluated with the same classifiers, splits and seeds; for the four hypergraph-aware classifiers this is their standard semi-supervised use. Raw is a reference and is not included in the bold (best) and underlined (second-best) marks, which compare the five embedding methods. LogReg is deterministic given its input, so its standard deviation over seeds is zero for Raw.

\begin{table}[htbp!]
\caption{Accuracy (\%) per dataset and classifier (part 1)}\label{tab:app_clf_acc1}
\setlength{\tabcolsep}{1.6pt}\scriptsize
\begin{tabular}{@{}llcccccc@{}}
\toprule
Dataset & Classifier & Raw & HyperFuse & TriCL & SE-HSSL & VilLain & HypeBoy \\
\midrule
\multirow{6}{*}{20News} & LogReg & 78.65{\scriptsize$\pm$0.00} & \underline{79.71{\tiny$\pm$0.17}} & 77.22{\tiny$\pm$0.04} & 76.31{\tiny$\pm$0.27} & 79.20{\tiny$\pm$0.07} & \textbf{79.84{\tiny$\pm$0.09}} \\
 & MLP & 79.10{\scriptsize$\pm$0.18} & 79.48{\tiny$\pm$0.25} & \underline{79.55{\tiny$\pm$0.20}} & 78.76{\tiny$\pm$0.12} & 78.38{\tiny$\pm$0.33} & \textbf{79.80{\tiny$\pm$0.02}} \\
 & HGNN & 80.00{\scriptsize$\pm$0.08} & 79.12{\tiny$\pm$0.05} & \textbf{79.83{\tiny$\pm$0.11}} & \underline{79.62{\tiny$\pm$0.13}} & 74.82{\tiny$\pm$0.09} & 77.76{\tiny$\pm$0.37} \\
 & HNHN & 79.60{\scriptsize$\pm$0.05} & \underline{79.78{\tiny$\pm$0.17}} & \textbf{80.07{\tiny$\pm$0.06}} & 79.65{\tiny$\pm$0.08} & 79.33{\tiny$\pm$0.28} & 79.74{\tiny$\pm$0.34} \\
 & UniGCNII & 80.10{\scriptsize$\pm$0.08} & 79.73{\tiny$\pm$0.05} & \textbf{80.19{\tiny$\pm$0.07}} & \underline{79.84{\tiny$\pm$0.06}} & 78.14{\tiny$\pm$0.08} & 78.64{\tiny$\pm$0.18} \\
 & AllSetTr. & 79.62{\scriptsize$\pm$0.21} & 79.94{\tiny$\pm$0.20} & \textbf{80.31{\tiny$\pm$0.15}} & 79.81{\tiny$\pm$0.36} & \underline{80.23{\tiny$\pm$0.08}} & 80.05{\tiny$\pm$0.18} \\
\midrule
\multirow{6}{*}{Citeseer} & LogReg & 55.91{\scriptsize$\pm$0.00} & \underline{70.23{\tiny$\pm$0.55}} & \textbf{70.34{\tiny$\pm$0.08}} & 69.02{\tiny$\pm$1.48} & 54.92{\tiny$\pm$0.29} & 68.63{\tiny$\pm$0.40} \\
 & MLP & 42.94{\scriptsize$\pm$2.88} & \textbf{71.42{\tiny$\pm$0.80}} & \underline{70.23{\tiny$\pm$0.72}} & 70.12{\tiny$\pm$1.60} & 51.71{\tiny$\pm$0.62} & 65.95{\tiny$\pm$6.88} \\
 & HGNN & 67.34{\scriptsize$\pm$0.67} & \textbf{71.29{\tiny$\pm$0.63}} & \underline{71.22{\tiny$\pm$0.09}} & 71.14{\tiny$\pm$0.34} & 55.03{\tiny$\pm$0.48} & 70.21{\tiny$\pm$0.35} \\
 & HNHN & 62.71{\scriptsize$\pm$0.29} & \textbf{71.19{\tiny$\pm$0.82}} & \underline{70.45{\tiny$\pm$1.25}} & 70.36{\tiny$\pm$0.81} & 55.83{\tiny$\pm$0.19} & 69.98{\tiny$\pm$0.49} \\
 & UniGCNII & 69.59{\scriptsize$\pm$0.38} & \textbf{71.55{\tiny$\pm$0.98}} & \underline{71.34{\tiny$\pm$0.71}} & 71.20{\tiny$\pm$0.88} & 55.86{\tiny$\pm$0.50} & 69.65{\tiny$\pm$1.03} \\
 & AllSetTr. & 66.36{\scriptsize$\pm$1.08} & \textbf{71.47{\tiny$\pm$0.44}} & \underline{70.85{\tiny$\pm$0.11}} & 70.69{\tiny$\pm$1.64} & 56.26{\tiny$\pm$1.05} & 68.81{\tiny$\pm$0.70} \\
\midrule
\multirow{6}{*}{Cora-CA} & LogReg & 58.96{\scriptsize$\pm$0.00} & \textbf{81.11{\tiny$\pm$0.41}} & \underline{80.95{\tiny$\pm$0.08}} & 80.00{\tiny$\pm$0.45} & 57.80{\tiny$\pm$0.42} & 77.73{\tiny$\pm$0.47} \\
 & MLP & 45.66{\scriptsize$\pm$0.84} & 78.31{\tiny$\pm$1.51} & \underline{78.64{\tiny$\pm$1.50}} & \textbf{79.23{\tiny$\pm$1.29}} & 52.11{\tiny$\pm$2.62} & 76.24{\tiny$\pm$0.48} \\
 & HGNN & 77.73{\scriptsize$\pm$0.31} & 79.70{\tiny$\pm$0.88} & \textbf{80.24{\tiny$\pm$0.24}} & \underline{79.84{\tiny$\pm$0.44}} & 57.29{\tiny$\pm$1.32} & 76.34{\tiny$\pm$0.36} \\
 & HNHN & 66.00{\scriptsize$\pm$0.96} & \underline{79.79{\tiny$\pm$0.67}} & \textbf{79.91{\tiny$\pm$0.62}} & 79.65{\tiny$\pm$0.21} & 58.79{\tiny$\pm$0.35} & 77.28{\tiny$\pm$0.33} \\
 & UniGCNII & 79.01{\scriptsize$\pm$0.47} & \textbf{80.99{\tiny$\pm$0.51}} & \underline{80.33{\tiny$\pm$0.15}} & 79.90{\tiny$\pm$0.27} & 59.27{\tiny$\pm$0.40} & 77.11{\tiny$\pm$0.45} \\
 & AllSetTr. & 74.88{\scriptsize$\pm$1.29} & \underline{80.89{\tiny$\pm$0.47}} & \textbf{81.13{\tiny$\pm$0.37}} & 80.68{\tiny$\pm$0.74} & 58.81{\tiny$\pm$0.95} & 77.80{\tiny$\pm$0.11} \\
\midrule
\multirow{6}{*}{Cora-CC} & LogReg & 56.59{\scriptsize$\pm$0.00} & 80.18{\tiny$\pm$0.36} & \textbf{80.93{\tiny$\pm$0.46}} & \underline{80.69{\tiny$\pm$0.64}} & 64.66{\tiny$\pm$0.93} & 78.47{\tiny$\pm$0.48} \\
 & MLP & 39.78{\scriptsize$\pm$10.37} & 73.30{\tiny$\pm$2.11} & \underline{77.61{\tiny$\pm$3.12}} & \textbf{79.36{\tiny$\pm$0.88}} & 56.56{\tiny$\pm$4.11} & 74.39{\tiny$\pm$2.66} \\
 & HGNN & 76.80{\scriptsize$\pm$0.73} & 79.18{\tiny$\pm$0.27} & \underline{80.11{\tiny$\pm$0.70}} & \textbf{80.21{\tiny$\pm$0.28}} & 73.64{\tiny$\pm$0.15} & 78.53{\tiny$\pm$0.26} \\
 & HNHN & 70.94{\scriptsize$\pm$0.54} & 80.41{\tiny$\pm$0.09} & \underline{80.54{\tiny$\pm$0.64}} & \textbf{80.81{\tiny$\pm$0.77}} & 73.90{\tiny$\pm$0.18} & 78.20{\tiny$\pm$0.27} \\
 & UniGCNII & 77.77{\scriptsize$\pm$0.18} & \textbf{80.95{\tiny$\pm$0.16}} & \underline{80.77{\tiny$\pm$0.22}} & 80.52{\tiny$\pm$0.41} & 74.54{\tiny$\pm$0.03} & 78.50{\tiny$\pm$0.57} \\
 & AllSetTr. & 73.75{\scriptsize$\pm$2.17} & \textbf{80.89{\tiny$\pm$0.10}} & 80.57{\tiny$\pm$0.10} & \underline{80.84{\tiny$\pm$0.60}} & 74.67{\tiny$\pm$0.49} & 78.35{\tiny$\pm$0.84} \\
\midrule
\multirow{6}{*}{DBLP} & LogReg & 78.54{\scriptsize$\pm$0.00} & 89.39{\tiny$\pm$0.06} & \underline{89.99{\tiny$\pm$0.12}} & OOT & 75.08{\tiny$\pm$0.16} & \textbf{90.27{\tiny$\pm$0.02}} \\
 & MLP & 80.76{\scriptsize$\pm$0.03} & 88.88{\tiny$\pm$0.15} & \textbf{91.10{\tiny$\pm$0.06}} & OOT & 78.23{\tiny$\pm$0.65} & \underline{89.57{\tiny$\pm$0.20}} \\
 & HGNN & 89.77{\scriptsize$\pm$0.04} & 86.80{\tiny$\pm$0.16} & \textbf{90.50{\tiny$\pm$0.02}} & OOT & 77.27{\tiny$\pm$0.12} & \underline{88.78{\tiny$\pm$0.01}} \\
 & HNHN & 84.61{\scriptsize$\pm$0.31} & 89.23{\tiny$\pm$0.18} & \textbf{91.22{\tiny$\pm$0.05}} & OOT & 79.06{\tiny$\pm$0.16} & \underline{89.61{\tiny$\pm$0.09}} \\
 & UniGCNII & 90.16{\scriptsize$\pm$0.06} & 88.26{\tiny$\pm$0.09} & \textbf{91.02{\tiny$\pm$0.07}} & OOT & 79.63{\tiny$\pm$0.08} & \underline{89.34{\tiny$\pm$0.04}} \\
 & AllSetTr. & 89.36{\scriptsize$\pm$0.42} & 89.77{\tiny$\pm$0.07} & \textbf{91.31{\tiny$\pm$0.05}} & OOT & 79.49{\tiny$\pm$0.15} & \underline{90.02{\tiny$\pm$0.07}} \\
\botrule
\end{tabular}
\end{table}

\begin{table}[htbp!]
\caption{Accuracy (\%) per dataset and classifier (part 2)}\label{tab:app_clf_acc2}
\setlength{\tabcolsep}{1.6pt}\scriptsize
\begin{tabular}{@{}llcccccc@{}}
\toprule
Dataset & Classifier & Raw & HyperFuse & TriCL & SE-HSSL & VilLain & HypeBoy \\
\midrule
\multirow{6}{*}{ModelNet40} & LogReg & 95.63{\scriptsize$\pm$0.00} & \textbf{97.18{\tiny$\pm$0.03}} & \underline{96.96{\tiny$\pm$0.08}} & 96.95{\tiny$\pm$0.05} & 86.54{\tiny$\pm$0.10} & 93.75{\tiny$\pm$0.06} \\
 & MLP & 94.53{\scriptsize$\pm$0.13} & 96.02{\tiny$\pm$0.23} & \underline{96.79{\tiny$\pm$0.19}} & \textbf{96.83{\tiny$\pm$0.22}} & 85.97{\tiny$\pm$0.49} & 91.83{\tiny$\pm$0.43} \\
 & HGNN & 94.06{\scriptsize$\pm$0.02} & \textbf{93.61{\tiny$\pm$0.10}} & \underline{92.56{\tiny$\pm$0.06}} & 92.46{\tiny$\pm$0.06} & 87.04{\tiny$\pm$0.16} & 88.93{\tiny$\pm$0.10} \\
 & HNHN & 96.18{\scriptsize$\pm$0.20} & \textbf{96.73{\tiny$\pm$0.16}} & \underline{94.99{\tiny$\pm$0.17}} & 94.63{\tiny$\pm$0.31} & 86.71{\tiny$\pm$0.06} & 90.90{\tiny$\pm$0.15} \\
 & UniGCNII & 96.59{\scriptsize$\pm$0.01} & \textbf{95.80{\tiny$\pm$0.05}} & \underline{93.76{\tiny$\pm$0.09}} & 93.71{\tiny$\pm$0.06} & 87.00{\tiny$\pm$0.03} & 90.00{\tiny$\pm$0.17} \\
 & AllSetTr. & 97.04{\scriptsize$\pm$0.18} & \textbf{97.15{\tiny$\pm$0.07}} & \underline{95.55{\tiny$\pm$0.13}} & 95.51{\tiny$\pm$0.11} & 87.05{\tiny$\pm$0.02} & 92.69{\tiny$\pm$0.31} \\
\midrule
\multirow{6}{*}{NTU2012} & LogReg & 70.97{\scriptsize$\pm$0.00} & \textbf{74.91{\tiny$\pm$0.36}} & \underline{74.66{\tiny$\pm$0.17}} & 74.05{\tiny$\pm$0.57} & 62.09{\tiny$\pm$0.48} & 69.22{\tiny$\pm$0.16} \\
 & MLP & 69.18{\scriptsize$\pm$0.26} & \textbf{72.88{\tiny$\pm$0.56}} & \underline{72.83{\tiny$\pm$0.38}} & 72.58{\tiny$\pm$0.19} & 61.15{\tiny$\pm$0.48} & 56.98{\tiny$\pm$0.77} \\
 & HGNN & 73.01{\scriptsize$\pm$0.35} & 72.35{\tiny$\pm$0.69} & \underline{73.12{\tiny$\pm$0.39}} & \textbf{73.23{\tiny$\pm$0.42}} & 69.48{\tiny$\pm$0.20} & 65.50{\tiny$\pm$0.75} \\
 & HNHN & 62.84{\scriptsize$\pm$0.47} & \underline{69.14{\tiny$\pm$0.28}} & \textbf{70.53{\tiny$\pm$0.38}} & 67.89{\tiny$\pm$1.17} & 68.07{\tiny$\pm$0.41} & 64.85{\tiny$\pm$1.41} \\
 & UniGCNII & 73.00{\scriptsize$\pm$0.66} & \textbf{73.57{\tiny$\pm$0.35}} & \underline{73.09{\tiny$\pm$0.43}} & 72.21{\tiny$\pm$0.29} & 69.74{\tiny$\pm$0.32} & 68.58{\tiny$\pm$0.52} \\
 & AllSetTr. & 72.32{\scriptsize$\pm$0.37} & \underline{73.25{\tiny$\pm$0.44}} & \textbf{73.99{\tiny$\pm$0.75}} & 73.08{\tiny$\pm$0.31} & 69.70{\tiny$\pm$0.38} & 70.53{\tiny$\pm$0.18} \\
\midrule
\multirow{6}{*}{PubMed} & LogReg & 73.46{\scriptsize$\pm$0.00} & \textbf{84.01{\tiny$\pm$0.50}} & 81.15{\tiny$\pm$0.21} & 79.61{\tiny$\pm$0.44} & 78.12{\tiny$\pm$0.22} & \underline{82.80{\tiny$\pm$0.18}} \\
 & MLP & 71.07{\scriptsize$\pm$1.07} & \underline{82.03{\tiny$\pm$0.58}} & \textbf{82.44{\tiny$\pm$0.35}} & 80.58{\tiny$\pm$0.33} & 77.33{\tiny$\pm$1.18} & 78.17{\tiny$\pm$1.58} \\
 & HGNN & 83.81{\scriptsize$\pm$0.06} & 82.34{\tiny$\pm$0.29} & \textbf{84.07{\tiny$\pm$0.19}} & \underline{83.65{\tiny$\pm$0.52}} & 81.71{\tiny$\pm$0.07} & 80.99{\tiny$\pm$0.54} \\
 & HNHN & 80.76{\scriptsize$\pm$0.33} & 82.85{\tiny$\pm$0.25} & \underline{83.76{\tiny$\pm$0.50}} & \textbf{83.98{\tiny$\pm$0.45}} & 81.78{\tiny$\pm$0.18} & 81.03{\tiny$\pm$0.22} \\
 & UniGCNII & 84.02{\scriptsize$\pm$0.10} & 84.00{\tiny$\pm$0.12} & \underline{84.20{\tiny$\pm$0.36}} & \textbf{84.26{\tiny$\pm$0.26}} & 81.59{\tiny$\pm$0.33} & 81.64{\tiny$\pm$0.23} \\
 & AllSetTr. & 81.97{\scriptsize$\pm$2.02} & 83.45{\tiny$\pm$0.35} & \underline{84.05{\tiny$\pm$0.24}} & \textbf{84.22{\tiny$\pm$0.15}} & 82.14{\tiny$\pm$0.33} & 81.77{\tiny$\pm$0.52} \\
\midrule
\multirow{6}{*}{Zoo} & LogReg & 74.90{\scriptsize$\pm$0.00} & 75.58{\tiny$\pm$0.63} & \textbf{78.46{\tiny$\pm$0.63}} & \underline{76.82{\tiny$\pm$1.04}} & 34.15{\tiny$\pm$1.09} & 72.56{\tiny$\pm$1.04} \\
 & MLP & 73.39{\scriptsize$\pm$0.47} & 75.45{\tiny$\pm$0.24} & \textbf{77.78{\tiny$\pm$0.71}} & \underline{76.27{\tiny$\pm$1.18}} & 29.36{\tiny$\pm$2.11} & 72.98{\tiny$\pm$0.63} \\
 & HGNN & 65.29{\scriptsize$\pm$2.89} & \textbf{74.48{\tiny$\pm$1.79}} & \underline{66.81{\tiny$\pm$2.11}} & 65.57{\tiny$\pm$3.20} & 46.78{\tiny$\pm$1.66} & 64.33{\tiny$\pm$2.61} \\
 & HNHN & 68.59{\scriptsize$\pm$1.56} & \textbf{70.92{\tiny$\pm$2.74}} & 68.45{\tiny$\pm$2.48} & \underline{69.55{\tiny$\pm$0.41}} & 63.38{\tiny$\pm$2.14} & 65.30{\tiny$\pm$4.77} \\
 & UniGCNII & 71.74{\scriptsize$\pm$1.90} & \textbf{72.15{\tiny$\pm$2.64}} & \underline{69.00{\tiny$\pm$1.26}} & 68.18{\tiny$\pm$2.03} & 66.12{\tiny$\pm$1.45} & 65.43{\tiny$\pm$1.24} \\
 & AllSetTr. & 67.90{\scriptsize$\pm$0.82} & \underline{72.97{\tiny$\pm$1.90}} & 71.74{\tiny$\pm$1.04} & 69.82{\tiny$\pm$1.45} & 70.37{\tiny$\pm$2.97} & \textbf{73.12{\tiny$\pm$3.09}} \\
\botrule
\end{tabular}
\end{table}

\begin{table}[htbp!]
\caption{Macro-F1 (\%) per dataset and classifier (part 1)}\label{tab:app_clf_f11}
\setlength{\tabcolsep}{1.6pt}\scriptsize
\begin{tabular}{@{}llcccccc@{}}
\toprule
Dataset & Classifier & Raw & HyperFuse & TriCL & SE-HSSL & VilLain & HypeBoy \\
\midrule
\multirow{6}{*}{20News} & LogReg & 76.25{\scriptsize$\pm$0.00} & \underline{77.37{\tiny$\pm$0.22}} & 74.84{\tiny$\pm$0.07} & 74.03{\tiny$\pm$0.29} & 76.52{\tiny$\pm$0.05} & \textbf{77.40{\tiny$\pm$0.06}} \\
 & MLP & 76.57{\scriptsize$\pm$0.13} & 76.85{\tiny$\pm$0.48} & \underline{76.91{\tiny$\pm$0.25}} & 76.36{\tiny$\pm$0.17} & 75.73{\tiny$\pm$0.50} & \textbf{77.23{\tiny$\pm$0.05}} \\
 & HGNN & 77.28{\scriptsize$\pm$0.14} & 75.98{\tiny$\pm$0.12} & \textbf{76.81{\tiny$\pm$0.23}} & \underline{76.67{\tiny$\pm$0.17}} & 67.39{\tiny$\pm$0.29} & 73.61{\tiny$\pm$0.83} \\
 & HNHN & 77.18{\scriptsize$\pm$0.05} & \underline{77.26{\tiny$\pm$0.21}} & \textbf{77.47{\tiny$\pm$0.14}} & 77.19{\tiny$\pm$0.12} & 76.39{\tiny$\pm$0.37} & 76.96{\tiny$\pm$0.48} \\
 & UniGCNII & 77.67{\scriptsize$\pm$0.03} & 76.96{\tiny$\pm$0.03} & \textbf{77.53{\tiny$\pm$0.04}} & \underline{77.32{\tiny$\pm$0.24}} & 74.11{\tiny$\pm$0.13} & 75.28{\tiny$\pm$0.26} \\
 & AllSetTr. & 77.24{\scriptsize$\pm$0.39} & 77.44{\tiny$\pm$0.24} & \textbf{78.01{\tiny$\pm$0.24}} & 77.46{\tiny$\pm$0.37} & \underline{77.70{\tiny$\pm$0.14}} & 77.47{\tiny$\pm$0.15} \\
\midrule
\multirow{6}{*}{Citeseer} & LogReg & 49.84{\scriptsize$\pm$0.00} & \textbf{64.57{\tiny$\pm$0.50}} & \underline{63.90{\tiny$\pm$0.28}} & 63.53{\tiny$\pm$1.29} & 49.17{\tiny$\pm$0.05} & 62.94{\tiny$\pm$0.34} \\
 & MLP & 38.73{\scriptsize$\pm$2.15} & \textbf{63.45{\tiny$\pm$0.70}} & 62.14{\tiny$\pm$1.27} & \underline{62.78{\tiny$\pm$0.92}} & 43.67{\tiny$\pm$0.34} & 58.22{\tiny$\pm$5.97} \\
 & HGNN & 60.45{\scriptsize$\pm$0.31} & \underline{64.02{\tiny$\pm$1.22}} & \textbf{64.61{\tiny$\pm$0.65}} & 63.48{\tiny$\pm$1.66} & 48.73{\tiny$\pm$0.73} & 61.89{\tiny$\pm$0.60} \\
 & HNHN & 55.62{\scriptsize$\pm$0.86} & \underline{62.73{\tiny$\pm$0.52}} & \textbf{62.75{\tiny$\pm$1.52}} & 62.48{\tiny$\pm$0.60} & 48.68{\tiny$\pm$1.14} & 61.84{\tiny$\pm$0.49} \\
 & UniGCNII & 63.03{\scriptsize$\pm$0.35} & \textbf{64.41{\tiny$\pm$0.35}} & 64.13{\tiny$\pm$0.65} & \underline{64.20{\tiny$\pm$0.21}} & 48.62{\tiny$\pm$2.04} & 62.89{\tiny$\pm$0.89} \\
 & AllSetTr. & 59.87{\scriptsize$\pm$1.05} & \textbf{64.26{\tiny$\pm$1.36}} & \underline{64.15{\tiny$\pm$0.60}} & 63.70{\tiny$\pm$0.54} & 49.07{\tiny$\pm$0.54} & 61.80{\tiny$\pm$0.60} \\
\midrule
\multirow{6}{*}{Cora-CA} & LogReg & 54.85{\scriptsize$\pm$0.00} & \textbf{79.47{\tiny$\pm$0.64}} & \underline{79.42{\tiny$\pm$0.15}} & 78.58{\tiny$\pm$0.37} & 54.69{\tiny$\pm$0.52} & 76.17{\tiny$\pm$0.35} \\
 & MLP & 40.77{\scriptsize$\pm$0.43} & 75.81{\tiny$\pm$2.43} & \underline{76.97{\tiny$\pm$1.48}} & \textbf{77.45{\tiny$\pm$1.55}} & 46.53{\tiny$\pm$3.19} & 73.64{\tiny$\pm$0.45} \\
 & HGNN & 75.55{\scriptsize$\pm$0.52} & 77.82{\tiny$\pm$1.21} & \textbf{78.76{\tiny$\pm$0.32}} & \underline{78.31{\tiny$\pm$0.51}} & 53.90{\tiny$\pm$1.39} & 73.86{\tiny$\pm$0.89} \\
 & HNHN & 61.82{\scriptsize$\pm$1.77} & 77.48{\tiny$\pm$1.18} & \underline{77.55{\tiny$\pm$0.99}} & \textbf{77.70{\tiny$\pm$0.25}} & 54.48{\tiny$\pm$1.04} & 75.34{\tiny$\pm$0.36} \\
 & UniGCNII & 77.37{\scriptsize$\pm$0.57} & \textbf{79.37{\tiny$\pm$0.62}} & \underline{78.85{\tiny$\pm$0.29}} & 78.20{\tiny$\pm$0.35} & 56.38{\tiny$\pm$0.53} & 75.41{\tiny$\pm$0.52} \\
 & AllSetTr. & 71.44{\scriptsize$\pm$1.56} & \underline{79.38{\tiny$\pm$0.42}} & \textbf{79.42{\tiny$\pm$0.83}} & 79.09{\tiny$\pm$0.71} & 55.53{\tiny$\pm$1.09} & 75.93{\tiny$\pm$0.13} \\
\midrule
\multirow{6}{*}{Cora-CC} & LogReg & 50.53{\scriptsize$\pm$0.00} & 78.35{\tiny$\pm$0.54} & \textbf{79.15{\tiny$\pm$0.56}} & \underline{78.92{\tiny$\pm$0.53}} & 62.98{\tiny$\pm$1.12} & 76.76{\tiny$\pm$0.55} \\
 & MLP & 35.47{\scriptsize$\pm$10.11} & 67.83{\tiny$\pm$2.56} & \underline{75.20{\tiny$\pm$4.43}} & \textbf{77.22{\tiny$\pm$1.30}} & 55.31{\tiny$\pm$4.23} & 70.24{\tiny$\pm$4.12} \\
 & HGNN & 74.42{\scriptsize$\pm$0.79} & 76.54{\tiny$\pm$1.04} & \underline{78.56{\tiny$\pm$1.04}} & \textbf{78.78{\tiny$\pm$0.33}} & 71.55{\tiny$\pm$0.08} & 77.14{\tiny$\pm$0.40} \\
 & HNHN & 67.74{\scriptsize$\pm$0.81} & 77.51{\tiny$\pm$0.08} & \underline{78.45{\tiny$\pm$0.65}} & \textbf{78.48{\tiny$\pm$1.27}} & 71.20{\tiny$\pm$1.54} & 75.97{\tiny$\pm$0.45} \\
 & UniGCNII & 75.42{\scriptsize$\pm$0.18} & 78.70{\tiny$\pm$0.39} & \textbf{78.94{\tiny$\pm$0.27}} & \underline{78.79{\tiny$\pm$0.56}} & 72.05{\tiny$\pm$0.98} & 76.93{\tiny$\pm$0.59} \\
 & AllSetTr. & 70.65{\scriptsize$\pm$2.30} & \underline{78.80{\tiny$\pm$0.15}} & 78.51{\tiny$\pm$0.26} & \textbf{79.07{\tiny$\pm$0.68}} & 72.17{\tiny$\pm$0.68} & 76.46{\tiny$\pm$1.30} \\
\midrule
\multirow{6}{*}{DBLP} & LogReg & 77.37{\scriptsize$\pm$0.00} & 88.86{\tiny$\pm$0.11} & \underline{89.53{\tiny$\pm$0.13}} & OOT & 72.93{\tiny$\pm$0.20} & \textbf{89.89{\tiny$\pm$0.01}} \\
 & MLP & 79.58{\scriptsize$\pm$0.02} & 88.25{\tiny$\pm$0.21} & \textbf{90.74{\tiny$\pm$0.04}} & OOT & 76.62{\tiny$\pm$0.78} & \underline{89.15{\tiny$\pm$0.25}} \\
 & HGNN & 89.37{\scriptsize$\pm$0.06} & 85.93{\tiny$\pm$0.14} & \textbf{90.18{\tiny$\pm$0.02}} & OOT & 75.45{\tiny$\pm$0.17} & \underline{88.36{\tiny$\pm$0.03}} \\
 & HNHN & 83.74{\scriptsize$\pm$0.45} & 88.60{\tiny$\pm$0.21} & \textbf{90.88{\tiny$\pm$0.05}} & OOT & 77.60{\tiny$\pm$0.12} & \underline{89.22{\tiny$\pm$0.11}} \\
 & UniGCNII & 89.79{\scriptsize$\pm$0.05} & 87.57{\tiny$\pm$0.09} & \textbf{90.72{\tiny$\pm$0.09}} & OOT & 78.20{\tiny$\pm$0.14} & \underline{88.93{\tiny$\pm$0.03}} \\
 & AllSetTr. & 88.88{\scriptsize$\pm$0.47} & 89.27{\tiny$\pm$0.08} & \textbf{90.98{\tiny$\pm$0.05}} & OOT & 78.04{\tiny$\pm$0.10} & \underline{89.66{\tiny$\pm$0.06}} \\
\botrule
\end{tabular}
\end{table}

\begin{table}[htbp!]
\caption{Macro-F1 (\%) per dataset and classifier (part 2)}\label{tab:app_clf_f12}
\setlength{\tabcolsep}{1.6pt}\scriptsize
\begin{tabular}{@{}llcccccc@{}}
\toprule
Dataset & Classifier & Raw & HyperFuse & TriCL & SE-HSSL & VilLain & HypeBoy \\
\midrule
\multirow{6}{*}{ModelNet40} & LogReg & 94.19{\scriptsize$\pm$0.00} & \textbf{96.00{\tiny$\pm$0.07}} & \underline{95.97{\tiny$\pm$0.10}} & 95.96{\tiny$\pm$0.03} & 83.27{\tiny$\pm$0.12} & 91.58{\tiny$\pm$0.07} \\
 & MLP & 91.78{\scriptsize$\pm$0.22} & 93.21{\tiny$\pm$0.46} & \textbf{95.53{\tiny$\pm$0.37}} & \underline{95.51{\tiny$\pm$0.60}} & 81.60{\tiny$\pm$0.80} & 87.62{\tiny$\pm$1.35} \\
 & HGNN & 91.88{\scriptsize$\pm$0.09} & \textbf{91.19{\tiny$\pm$0.15}} & \underline{89.86{\tiny$\pm$0.18}} & 89.59{\tiny$\pm$0.25} & 82.90{\tiny$\pm$0.36} & 82.18{\tiny$\pm$0.15} \\
 & HNHN & 94.45{\scriptsize$\pm$0.34} & \textbf{95.10{\tiny$\pm$0.23}} & \underline{92.92{\tiny$\pm$0.29}} & 92.55{\tiny$\pm$0.48} & 82.37{\tiny$\pm$0.21} & 86.93{\tiny$\pm$0.38} \\
 & UniGCNII & 95.31{\scriptsize$\pm$0.04} & \textbf{94.21{\tiny$\pm$0.02}} & \underline{91.63{\tiny$\pm$0.07}} & 91.60{\tiny$\pm$0.05} & 82.89{\tiny$\pm$0.17} & 85.30{\tiny$\pm$0.32} \\
 & AllSetTr. & 95.79{\scriptsize$\pm$0.25} & \textbf{95.85{\tiny$\pm$0.13}} & \underline{93.94{\tiny$\pm$0.27}} & 93.91{\tiny$\pm$0.21} & 82.95{\tiny$\pm$0.18} & 90.04{\tiny$\pm$0.46} \\
\midrule
\multirow{6}{*}{NTU2012} & LogReg & 58.14{\scriptsize$\pm$0.00} & 62.21{\tiny$\pm$0.44} & \textbf{63.56{\tiny$\pm$0.10}} & \underline{63.39{\tiny$\pm$0.71}} & 54.26{\tiny$\pm$0.25} & 56.89{\tiny$\pm$0.10} \\
 & MLP & 55.68{\scriptsize$\pm$0.33} & 59.27{\tiny$\pm$0.64} & \underline{61.41{\tiny$\pm$0.50}} & \textbf{61.41{\tiny$\pm$0.60}} & 54.17{\tiny$\pm$0.42} & 39.03{\tiny$\pm$1.34} \\
 & HGNN & 60.85{\scriptsize$\pm$0.49} & 58.78{\tiny$\pm$0.90} & \underline{60.51{\tiny$\pm$0.81}} & \textbf{61.03{\tiny$\pm$0.77}} & 56.66{\tiny$\pm$0.24} & 48.88{\tiny$\pm$1.34} \\
 & HNHN & 49.03{\scriptsize$\pm$0.65} & 55.59{\tiny$\pm$0.32} & \textbf{58.06{\tiny$\pm$0.17}} & 55.12{\tiny$\pm$1.25} & \underline{55.61{\tiny$\pm$0.45}} & 49.58{\tiny$\pm$1.82} \\
 & UniGCNII & 61.05{\scriptsize$\pm$0.47} & \underline{60.83{\tiny$\pm$0.71}} & \textbf{61.11{\tiny$\pm$0.61}} & 60.44{\tiny$\pm$0.43} & 57.03{\tiny$\pm$0.52} & 55.25{\tiny$\pm$0.36} \\
 & AllSetTr. & 59.60{\scriptsize$\pm$0.14} & 60.52{\tiny$\pm$0.36} & \textbf{62.09{\tiny$\pm$0.66}} & \underline{61.16{\tiny$\pm$0.24}} & 57.24{\tiny$\pm$0.49} & 58.05{\tiny$\pm$0.58} \\
\midrule
\multirow{6}{*}{PubMed} & LogReg & 72.58{\scriptsize$\pm$0.00} & \textbf{82.85{\tiny$\pm$0.58}} & 79.62{\tiny$\pm$0.29} & 77.88{\tiny$\pm$0.41} & 75.23{\tiny$\pm$0.25} & \underline{81.49{\tiny$\pm$0.12}} \\
 & MLP & 70.08{\scriptsize$\pm$0.94} & \underline{80.93{\tiny$\pm$0.59}} & \textbf{81.21{\tiny$\pm$0.59}} & 79.12{\tiny$\pm$0.26} & 73.73{\tiny$\pm$1.12} & 77.11{\tiny$\pm$1.73} \\
 & HGNN & 82.45{\scriptsize$\pm$0.09} & 81.01{\tiny$\pm$0.28} & \textbf{82.49{\tiny$\pm$0.24}} & \underline{81.87{\tiny$\pm$0.72}} & 79.14{\tiny$\pm$0.12} & 79.61{\tiny$\pm$0.52} \\
 & HNHN & 79.75{\scriptsize$\pm$0.30} & 81.70{\tiny$\pm$0.23} & \underline{82.34{\tiny$\pm$0.77}} & \textbf{82.50{\tiny$\pm$0.64}} & 79.37{\tiny$\pm$0.17} & 79.77{\tiny$\pm$0.28} \\
 & UniGCNII & 82.79{\scriptsize$\pm$0.13} & \textbf{82.79{\tiny$\pm$0.17}} & 82.73{\tiny$\pm$0.44} & \underline{82.77{\tiny$\pm$0.27}} & 79.11{\tiny$\pm$0.40} & 80.24{\tiny$\pm$0.29} \\
 & AllSetTr. & 80.95{\scriptsize$\pm$1.94} & 82.05{\tiny$\pm$0.39} & \underline{82.60{\tiny$\pm$0.31}} & \textbf{82.78{\tiny$\pm$0.20}} & 79.73{\tiny$\pm$0.27} & 80.14{\tiny$\pm$0.55} \\
\midrule
\multirow{6}{*}{Zoo} & LogReg & 49.95{\scriptsize$\pm$0.00} & 44.69{\tiny$\pm$0.29} & \textbf{56.20{\tiny$\pm$0.32}} & \underline{51.99{\tiny$\pm$1.11}} & 12.62{\tiny$\pm$0.76} & 48.07{\tiny$\pm$0.81} \\
 & MLP & 50.12{\scriptsize$\pm$0.71} & 46.39{\tiny$\pm$2.69} & \textbf{55.08{\tiny$\pm$1.34}} & \underline{50.33{\tiny$\pm$0.99}} & 13.05{\tiny$\pm$2.68} & 48.08{\tiny$\pm$0.31} \\
 & HGNN & 31.59{\scriptsize$\pm$2.71} & \textbf{41.88{\tiny$\pm$4.16}} & \underline{33.85{\tiny$\pm$4.22}} & 32.05{\tiny$\pm$4.13} & 12.36{\tiny$\pm$0.42} & 28.51{\tiny$\pm$2.29} \\
 & HNHN & 40.01{\scriptsize$\pm$2.75} & \textbf{41.60{\tiny$\pm$4.42}} & 35.33{\tiny$\pm$4.21} & \underline{38.81{\tiny$\pm$2.03}} & 27.86{\tiny$\pm$1.20} & 37.00{\tiny$\pm$8.38} \\
 & UniGCNII & 42.77{\scriptsize$\pm$2.69} & \textbf{40.82{\tiny$\pm$5.83}} & 36.43{\tiny$\pm$2.85} & \underline{37.43{\tiny$\pm$2.54}} & 28.81{\tiny$\pm$0.46} & 34.94{\tiny$\pm$2.19} \\
 & AllSetTr. & 37.22{\scriptsize$\pm$1.80} & \textbf{43.31{\tiny$\pm$4.32}} & \underline{41.60{\tiny$\pm$4.00}} & 40.69{\tiny$\pm$3.29} & 39.35{\tiny$\pm$2.85} & 40.61{\tiny$\pm$3.76} \\
\botrule
\end{tabular}
\end{table}

\clearpage
\section{Ablation of the side input}\label{app:ablation}
\setcounter{table}{0}\setcounter{figure}{0}

We ablate the side input $[\mat{F}\mid\mat{G}]$ of the Stage-C encoder (Sect.~\ref{sec:stageC}); Stages~A and~B and the hyperedge weights are unchanged. Following the standard linear evaluation protocol \citep{velickovic2019dgi,lee2023tricl,li2024sehssl}, we evaluate the variants with logistic regression on the frozen embeddings, which measures how linearly separable the classes are in the embedding itself.

Table~\ref{tab:abl_variants} lists the encoder inputs of each variant.

\begin{table}[htbp!]
\caption{Encoder inputs of the ablation variants. Stages~A and~B and the hyperedge weights are used in all variants}\label{tab:abl_variants}
\begin{tabular}{@{}lccc@{}}
\toprule
Variant & Raw features $\hat{\mat{X}}$ & Feature summary $\mat{F}$ & Incidence summary $\mat{G}$ \\
\midrule
HyperFuse & \checkmark & \checkmark & \checkmark \\
No side input & \checkmark & -- & -- \\
Side F & \checkmark & \checkmark & -- \\
Side G & \checkmark & -- & \checkmark \\
Side only & -- & \checkmark & \checkmark \\
\botrule
\end{tabular}
\end{table}
With the full side input, HyperFuse was the most accurate variant on average (80.37\%, against 80.24\% without the side input, 80.14\% with $\mat{F}$ only and 79.75\% with $\mat{G}$ only; Table~\ref{tab:abl_side_acc}).

\begin{table}[htbp!]
\caption{Ablation of the side input $[\mat{F}\mid\mat{G}]$: accuracy (\%) with logistic regression. HyperFuse: full model (main results); No side input: raw features only; Side F / Side G: raw features and one block; Side only: $[\mat{F}\mid\mat{G}]$ only. Mean $\pm$ standard deviation over three seeds. Best in bold. Average: over the eight datasets other than DBLP}\label{tab:abl_side_acc}
\setlength{\tabcolsep}{3pt}\footnotesize
\begin{tabular}{@{}lccccc@{}}
\toprule
Dataset & HyperFuse & No side input & Side F & Side G & Side only \\
\midrule
20News & \textbf{79.71{\scriptsize$\pm$0.17}} & 79.45{\scriptsize$\pm$0.01} & 79.50{\scriptsize$\pm$0.05} & 79.59{\scriptsize$\pm$0.07} & 79.09{\scriptsize$\pm$0.13} \\
Citeseer & 70.23{\scriptsize$\pm$0.55} & \textbf{70.74{\scriptsize$\pm$0.41}} & 70.42{\scriptsize$\pm$0.60} & 68.25{\scriptsize$\pm$0.19} & 64.16{\scriptsize$\pm$0.19} \\
Cora-CA & \textbf{81.11{\scriptsize$\pm$0.41}} & 81.07{\scriptsize$\pm$0.38} & 80.91{\scriptsize$\pm$0.39} & 79.65{\scriptsize$\pm$0.13} & 73.72{\scriptsize$\pm$0.70} \\
Cora-CC & 80.18{\scriptsize$\pm$0.36} & 80.14{\scriptsize$\pm$0.43} & \textbf{80.81{\scriptsize$\pm$0.28}} & 80.09{\scriptsize$\pm$0.65} & 72.02{\scriptsize$\pm$1.11} \\
ModelNet40 & \textbf{97.18{\scriptsize$\pm$0.03}} & 96.83{\scriptsize$\pm$0.13} & 97.08{\scriptsize$\pm$0.08} & 96.97{\scriptsize$\pm$0.10} & 97.08{\scriptsize$\pm$0.02} \\
NTU2012 & \textbf{74.91{\scriptsize$\pm$0.36}} & 74.41{\scriptsize$\pm$0.14} & 74.37{\scriptsize$\pm$0.45} & 74.15{\scriptsize$\pm$0.17} & 73.67{\scriptsize$\pm$0.19} \\
PubMed & \textbf{84.01{\scriptsize$\pm$0.50}} & 83.56{\scriptsize$\pm$0.39} & 83.66{\scriptsize$\pm$0.07} & 83.74{\scriptsize$\pm$0.14} & 82.58{\scriptsize$\pm$0.26} \\
Zoo & 75.58{\scriptsize$\pm$0.63} & 75.72{\scriptsize$\pm$0.71} & 74.35{\scriptsize$\pm$0.86} & 75.58{\scriptsize$\pm$0.47} & \textbf{76.40{\scriptsize$\pm$0.47}} \\
\midrule
DBLP & 89.39{\scriptsize$\pm$0.06} & 89.36{\scriptsize$\pm$0.24} & \textbf{89.40{\scriptsize$\pm$0.05}} & 89.33{\scriptsize$\pm$0.21} & 85.97{\scriptsize$\pm$0.33} \\
\midrule
Average & \textbf{80.37} & 80.24 & 80.14 & 79.75 & 77.34 \\
\botrule
\end{tabular}
\end{table}

\begin{table}[htbp!]
\caption{Ablation of the side input $[\mat{F}\mid\mat{G}]$: macro-F1 (\%) with logistic regression. HyperFuse: full model (main results); No side input: raw features only; Side F / Side G: raw features and one block; Side only: $[\mat{F}\mid\mat{G}]$ only. Mean $\pm$ standard deviation over three seeds. Best in bold. Average: over the eight datasets other than DBLP}\label{tab:abl_side_f1}
\setlength{\tabcolsep}{3pt}\footnotesize
\begin{tabular}{@{}lccccc@{}}
\toprule
Dataset & HyperFuse & No side input & Side F & Side G & Side only \\
\midrule
20News & \textbf{77.37{\scriptsize$\pm$0.22}} & 77.02{\scriptsize$\pm$0.06} & 77.06{\scriptsize$\pm$0.04} & 77.20{\scriptsize$\pm$0.04} & 76.74{\scriptsize$\pm$0.13} \\
Citeseer & 64.57{\scriptsize$\pm$0.50} & \textbf{65.23{\scriptsize$\pm$0.26}} & 64.65{\scriptsize$\pm$0.97} & 62.34{\scriptsize$\pm$0.20} & 58.37{\scriptsize$\pm$0.18} \\
Cora-CA & \textbf{79.47{\scriptsize$\pm$0.64}} & 79.36{\scriptsize$\pm$0.51} & 79.11{\scriptsize$\pm$0.37} & 77.80{\scriptsize$\pm$0.09} & 71.64{\scriptsize$\pm$0.74} \\
Cora-CC & 78.35{\scriptsize$\pm$0.54} & 78.10{\scriptsize$\pm$0.73} & \textbf{78.85{\scriptsize$\pm$0.40}} & 77.90{\scriptsize$\pm$0.86} & 69.47{\scriptsize$\pm$1.05} \\
ModelNet40 & 96.00{\scriptsize$\pm$0.07} & 95.62{\scriptsize$\pm$0.18} & 95.93{\scriptsize$\pm$0.06} & 95.77{\scriptsize$\pm$0.12} & \textbf{96.18{\scriptsize$\pm$0.08}} \\
NTU2012 & \textbf{62.21{\scriptsize$\pm$0.44}} & 62.17{\scriptsize$\pm$0.35} & 61.96{\scriptsize$\pm$0.64} & 61.39{\scriptsize$\pm$0.22} & 61.28{\scriptsize$\pm$0.12} \\
PubMed & \textbf{82.85{\scriptsize$\pm$0.58}} & 82.50{\scriptsize$\pm$0.43} & 82.61{\scriptsize$\pm$0.13} & 82.55{\scriptsize$\pm$0.22} & 80.74{\scriptsize$\pm$0.32} \\
Zoo & 44.69{\scriptsize$\pm$0.29} & 44.65{\scriptsize$\pm$1.40} & 44.09{\scriptsize$\pm$1.05} & 46.46{\scriptsize$\pm$3.05} & \textbf{51.13{\scriptsize$\pm$2.15}} \\
\midrule
DBLP & \textbf{88.86{\scriptsize$\pm$0.11}} & 88.82{\scriptsize$\pm$0.25} & 88.85{\scriptsize$\pm$0.05} & 88.79{\scriptsize$\pm$0.23} & 85.12{\scriptsize$\pm$0.42} \\
\midrule
Average & \textbf{73.19} & 73.08 & 73.03 & 72.67 & 70.69 \\
\botrule
\end{tabular}
\end{table}

\clearpage
\section{Robustness to weak node features}\label{app:featdeg}
\setcounter{table}{0}\setcounter{figure}{0}

The Stage-C encoder, whose objective follows CCA-SSG \citep{zhang2021ccassg}, is on its own almost as accurate as HyperFuse when the full node features are available. To test what Stages~A and~B add, we compare HyperFuse with this encoder alone (raw features as its only input, uniform hyperedge weights) when the node features are degraded, since Stages~A and~B extract structural information that does not depend on the features. Both models receive the same degraded features: a random 100\%, 50\%, 20\% or 5\% of the feature columns, or 128 columns of Gaussian noise; the hypergraph, labels and splits are unchanged. For noise features, the automatic rule of Sect.~\ref{sec:setup} sets the feature dropout to $p_f=0$, so we also report $p_f=0.3$ (Noise$^{*}$). The embeddings are evaluated with logistic regression.

With the full features, the two models performed alike (80.37\% and 80.23\% average accuracy; Table~\ref{tab:fd_acc}). As features were removed, the difference grew to 5.23 points with 5\% of the columns and to 19.37 points with Noise$^{*}$, concentrated on ModelNet40, NTU2012, PubMed, Zoo and, with noise features, 20News. On the citation hypergraphs, the encoder alone was up to 5.5 points more accurate with 20\% and 5\% of the columns. Macro-F1 follows the same pattern (Table~\ref{tab:fd_f1}).

\begin{table}[htbp!]
\caption{Accuracy (\%) of logistic regression with degraded node features. Columns: fraction of feature columns kept, or Gaussian noise features (Noise: automatic feature dropout, $p_f=0$; Noise$^{*}$: $p_f=0.3$). Encoder: the Stage-C encoder alone, without Stages~A and~B. Mean $\pm$ standard deviation over three seeds. Better of the two models in bold}\label{tab:fd_acc}
\setlength{\tabcolsep}{2pt}\scriptsize
\begin{tabular}{@{}l@{\hspace{6pt}}lcccccc@{}}
\toprule
Dataset & Model & 100\% & 50\% & 20\% & 5\% & Noise & Noise$^{*}$ \\
\midrule
\multirow{2}{*}{20News} & HyperFuse & \textbf{79.71{\tiny$\pm$0.17}} & 77.46{\tiny$\pm$0.66} & 77.82{\tiny$\pm$0.33} & \textbf{78.39{\tiny$\pm$0.01}} & \textbf{29.57{\tiny$\pm$0.84}} & \textbf{75.35{\tiny$\pm$0.78}} \\
 & Encoder & 79.49{\tiny$\pm$0.16} & \textbf{77.56{\tiny$\pm$0.21}} & \textbf{78.02{\tiny$\pm$0.29}} & 78.12{\tiny$\pm$0.37} & 28.20{\tiny$\pm$0.47} & 43.73{\tiny$\pm$0.77} \\
\midrule
\multirow{2}{*}{Citeseer} & HyperFuse & 70.23{\tiny$\pm$0.55} & \textbf{67.00{\tiny$\pm$2.03}} & 59.40{\tiny$\pm$1.61} & 42.14{\tiny$\pm$2.62} & 18.95{\tiny$\pm$0.63} & 23.41{\tiny$\pm$0.53} \\
 & Encoder & \textbf{70.75{\tiny$\pm$0.53}} & 66.76{\tiny$\pm$1.69} & \textbf{60.85{\tiny$\pm$0.94}} & \textbf{47.68{\tiny$\pm$1.48}} & \textbf{19.45{\tiny$\pm$0.73}} & \textbf{24.13{\tiny$\pm$1.08}} \\
\midrule
\multirow{2}{*}{Cora-CA} & HyperFuse & 81.11{\tiny$\pm$0.41} & 73.44{\tiny$\pm$3.58} & 61.39{\tiny$\pm$3.83} & 43.22{\tiny$\pm$0.99} & 17.96{\tiny$\pm$0.55} & \textbf{23.55{\tiny$\pm$0.96}} \\
 & Encoder & \textbf{81.23{\tiny$\pm$0.47}} & \textbf{74.26{\tiny$\pm$2.78}} & \textbf{64.14{\tiny$\pm$2.70}} & \textbf{46.03{\tiny$\pm$1.87}} & \textbf{18.60{\tiny$\pm$0.42}} & 21.78{\tiny$\pm$0.28} \\
\midrule
\multirow{2}{*}{Cora-CC} & HyperFuse & \textbf{80.18{\tiny$\pm$0.36}} & 73.93{\tiny$\pm$0.76} & 62.01{\tiny$\pm$1.57} & 48.18{\tiny$\pm$3.97} & \textbf{19.23{\tiny$\pm$1.51}} & 26.19{\tiny$\pm$0.76} \\
 & Encoder & 80.02{\tiny$\pm$0.44} & \textbf{75.00{\tiny$\pm$0.49}} & \textbf{64.98{\tiny$\pm$2.21}} & \textbf{50.53{\tiny$\pm$2.35}} & 18.24{\tiny$\pm$1.44} & \textbf{26.65{\tiny$\pm$1.15}} \\
\midrule
\multirow{2}{*}{ModelNet40} & HyperFuse & \textbf{97.18{\tiny$\pm$0.03}} & \textbf{96.60{\tiny$\pm$0.09}} & \textbf{93.95{\tiny$\pm$0.62}} & \textbf{87.89{\tiny$\pm$0.61}} & \textbf{6.55{\tiny$\pm$0.31}} & \textbf{53.00{\tiny$\pm$1.34}} \\
 & Encoder & 96.84{\tiny$\pm$0.12} & 96.01{\tiny$\pm$0.11} & 92.24{\tiny$\pm$0.63} & 81.75{\tiny$\pm$1.24} & 3.96{\tiny$\pm$0.03} & 5.66{\tiny$\pm$0.15} \\
\midrule
\multirow{2}{*}{NTU2012} & HyperFuse & \textbf{74.91{\tiny$\pm$0.36}} & \textbf{73.37{\tiny$\pm$0.81}} & \textbf{68.25{\tiny$\pm$0.80}} & \textbf{64.99{\tiny$\pm$0.70}} & \textbf{10.25{\tiny$\pm$0.36}} & \textbf{63.97{\tiny$\pm$0.79}} \\
 & Encoder & 74.40{\tiny$\pm$0.11} & 70.28{\tiny$\pm$0.10} & 59.91{\tiny$\pm$0.73} & 36.03{\tiny$\pm$1.83} & 3.33{\tiny$\pm$0.26} & 30.28{\tiny$\pm$0.49} \\
\midrule
\multirow{2}{*}{PubMed} & HyperFuse & \textbf{84.01{\tiny$\pm$0.50}} & \textbf{82.77{\tiny$\pm$0.38}} & \textbf{81.48{\tiny$\pm$0.90}} & \textbf{80.03{\tiny$\pm$1.49}} & \textbf{46.15{\tiny$\pm$0.43}} & \textbf{77.44{\tiny$\pm$0.65}} \\
 & Encoder & 83.63{\tiny$\pm$0.41} & 82.36{\tiny$\pm$0.32} & 79.92{\tiny$\pm$0.72} & 75.61{\tiny$\pm$3.30} & 36.28{\tiny$\pm$0.55} & 42.99{\tiny$\pm$1.26} \\
\midrule
\multirow{2}{*}{Zoo} & HyperFuse & \textbf{75.58{\tiny$\pm$0.63}} & \textbf{76.54{\tiny$\pm$0.41}} & \textbf{76.96{\tiny$\pm$2.57}} & \textbf{76.82{\tiny$\pm$3.45}} & \textbf{36.63{\tiny$\pm$2.29}} & \textbf{74.07{\tiny$\pm$3.30}} \\
 & Encoder & 75.45{\tiny$\pm$0.24} & 75.17{\tiny$\pm$2.80} & 62.82{\tiny$\pm$5.76} & 64.06{\tiny$\pm$3.82} & 36.35{\tiny$\pm$1.71} & 66.80{\tiny$\pm$3.80} \\
\midrule
\multirow{2}{*}{Average} & HyperFuse & \textbf{80.37} & \textbf{77.64} & \textbf{72.66} & \textbf{65.21} & \textbf{23.16} & \textbf{52.12} \\
 & Encoder & 80.23 & 77.18 & 70.36 & 59.98 & 20.55 & 32.75 \\
\botrule
\end{tabular}
\end{table}

\begin{table}[htbp!]
\caption{Macro-F1 (\%) of logistic regression with degraded node features. Columns: fraction of feature columns kept, or Gaussian noise features (Noise: automatic feature dropout, $p_f=0$; Noise$^{*}$: $p_f=0.3$). Encoder: the Stage-C encoder alone, without Stages~A and~B. Mean $\pm$ standard deviation over three seeds. Better of the two models in bold}\label{tab:fd_f1}
\setlength{\tabcolsep}{2pt}\scriptsize
\begin{tabular}{@{}l@{\hspace{6pt}}lcccccc@{}}
\toprule
Dataset & Model & 100\% & 50\% & 20\% & 5\% & Noise & Noise$^{*}$ \\
\midrule
\multirow{2}{*}{20News} & HyperFuse & \textbf{77.37{\tiny$\pm$0.22}} & 74.92{\tiny$\pm$0.54} & 75.13{\tiny$\pm$0.36} & \textbf{75.79{\tiny$\pm$0.16}} & \textbf{26.30{\tiny$\pm$0.92}} & \textbf{72.16{\tiny$\pm$0.92}} \\
 & Encoder & 77.06{\tiny$\pm$0.16} & \textbf{74.99{\tiny$\pm$0.15}} & \textbf{75.28{\tiny$\pm$0.27}} & 75.23{\tiny$\pm$0.39} & 24.85{\tiny$\pm$0.45} & 39.29{\tiny$\pm$1.31} \\
\midrule
\multirow{2}{*}{Citeseer} & HyperFuse & 64.57{\tiny$\pm$0.50} & \textbf{61.75{\tiny$\pm$2.54}} & 54.28{\tiny$\pm$1.90} & 37.49{\tiny$\pm$2.95} & 17.26{\tiny$\pm$0.79} & 21.44{\tiny$\pm$0.51} \\
 & Encoder & \textbf{65.31{\tiny$\pm$0.24}} & 61.64{\tiny$\pm$1.59} & \textbf{55.59{\tiny$\pm$0.91}} & \textbf{42.22{\tiny$\pm$1.20}} & \textbf{17.46{\tiny$\pm$0.61}} & \textbf{21.84{\tiny$\pm$0.88}} \\
\midrule
\multirow{2}{*}{Cora-CA} & HyperFuse & 79.47{\tiny$\pm$0.64} & 71.10{\tiny$\pm$3.93} & 58.81{\tiny$\pm$2.41} & 38.22{\tiny$\pm$0.72} & 14.62{\tiny$\pm$0.71} & \textbf{19.32{\tiny$\pm$0.92}} \\
 & Encoder & \textbf{79.58{\tiny$\pm$0.67}} & \textbf{72.28{\tiny$\pm$2.86}} & \textbf{61.60{\tiny$\pm$1.64}} & \textbf{41.07{\tiny$\pm$3.17}} & \textbf{15.21{\tiny$\pm$0.55}} & 17.62{\tiny$\pm$0.24} \\
\midrule
\multirow{2}{*}{Cora-CC} & HyperFuse & \textbf{78.35{\tiny$\pm$0.54}} & 71.80{\tiny$\pm$1.62} & 58.23{\tiny$\pm$1.61} & 44.52{\tiny$\pm$2.54} & \textbf{15.72{\tiny$\pm$1.17}} & 22.85{\tiny$\pm$0.49} \\
 & Encoder & 77.92{\tiny$\pm$0.68} & \textbf{73.12{\tiny$\pm$1.10}} & \textbf{61.55{\tiny$\pm$2.69}} & \textbf{46.66{\tiny$\pm$1.30}} & 14.95{\tiny$\pm$1.00} & \textbf{23.40{\tiny$\pm$1.39}} \\
\midrule
\multirow{2}{*}{ModelNet40} & HyperFuse & \textbf{96.00{\tiny$\pm$0.07}} & \textbf{95.28{\tiny$\pm$0.16}} & \textbf{92.17{\tiny$\pm$0.58}} & \textbf{83.60{\tiny$\pm$0.68}} & \textbf{4.71{\tiny$\pm$0.25}} & \textbf{44.26{\tiny$\pm$1.37}} \\
 & Encoder & 95.65{\tiny$\pm$0.18} & 94.56{\tiny$\pm$0.21} & 90.01{\tiny$\pm$0.49} & 75.61{\tiny$\pm$0.69} & 2.76{\tiny$\pm$0.03} & 3.55{\tiny$\pm$0.07} \\
\midrule
\multirow{2}{*}{NTU2012} & HyperFuse & \textbf{62.21{\tiny$\pm$0.44}} & \textbf{60.68{\tiny$\pm$1.25}} & \textbf{54.70{\tiny$\pm$1.47}} & \textbf{50.80{\tiny$\pm$1.21}} & \textbf{5.38{\tiny$\pm$0.24}} & \textbf{48.77{\tiny$\pm$1.14}} \\
 & Encoder & 62.10{\tiny$\pm$0.32} & 57.58{\tiny$\pm$0.31} & 45.83{\tiny$\pm$0.83} & 23.46{\tiny$\pm$0.32} & 1.25{\tiny$\pm$0.10} & 23.23{\tiny$\pm$0.68} \\
\midrule
\multirow{2}{*}{PubMed} & HyperFuse & \textbf{82.85{\tiny$\pm$0.58}} & \textbf{81.72{\tiny$\pm$0.43}} & \textbf{80.12{\tiny$\pm$1.02}} & \textbf{78.04{\tiny$\pm$1.75}} & \textbf{42.92{\tiny$\pm$0.55}} & \textbf{74.65{\tiny$\pm$0.66}} \\
 & Encoder & 82.57{\tiny$\pm$0.43} & 81.34{\tiny$\pm$0.30} & 78.81{\tiny$\pm$0.79} & 73.84{\tiny$\pm$3.27} & 33.70{\tiny$\pm$0.57} & 39.73{\tiny$\pm$1.40} \\
\midrule
\multirow{2}{*}{Zoo} & HyperFuse & \textbf{44.69{\tiny$\pm$0.29}} & 46.95{\tiny$\pm$2.85} & \textbf{51.35{\tiny$\pm$4.63}} & \textbf{51.92{\tiny$\pm$2.83}} & \textbf{10.27{\tiny$\pm$0.33}} & \textbf{44.60{\tiny$\pm$4.51}} \\
 & Encoder & 44.62{\tiny$\pm$1.29} & \textbf{48.07{\tiny$\pm$5.39}} & 34.27{\tiny$\pm$6.82} & 31.96{\tiny$\pm$4.09} & 8.87{\tiny$\pm$0.31} & 35.02{\tiny$\pm$4.64} \\
\midrule
\multirow{2}{*}{Average} & HyperFuse & \textbf{73.19} & \textbf{70.52} & \textbf{65.60} & \textbf{57.55} & \textbf{17.15} & \textbf{43.50} \\
 & Encoder & 73.10 & 70.45 & 62.87 & 51.26 & 14.88 & 25.46 \\
\botrule
\end{tabular}
\end{table}

\clearpage

\end{appendices}
\end{document}